\documentclass[11pt]{article} 

\usepackage[letterpaper,margin=1in]{geometry}

\title{Self-Directed Linear Classification}

\usepackage{wrapfig}

\newif\ifred
\redtrue 

\usepackage[T1]{fontenc}
\usepackage{microtype}

\usepackage{titlesec}
\titleformat*{\paragraph}{\bfseries}

\usepackage{amsthm,amsmath,amssymb,amsfonts,amssymb,mathtools}
\usepackage{amsmath,amssymb,amsfonts,amssymb,mathtools}
\usepackage{xcolor}
\usepackage{empheq}
\usepackage{dsfont}
\usepackage{mathrsfs}
\usepackage{graphicx}
\usepackage{enumitem}
\usepackage{aliascnt} 
\usepackage{xspace}

\ifred

\else

\fi

\usepackage{caption}
\usepackage{tikz}
\usetikzlibrary{patterns}
\usetikzlibrary{arrows,shapes,automata,backgrounds,petri,positioning}
\usetikzlibrary{shadows}
\usetikzlibrary{calc}
\usetikzlibrary{spy}
\usetikzlibrary{angles, quotes}
\usetikzlibrary{matrix}
\usepackage{pgf,pgfplots}
\usepackage{pgfmath,pgffor}
\pgfplotsset{compat=1.17}

\usepackage{framed}
\usepackage{algorithm2e}
\usepackage[noend]{algpseudocode}
\SetAlCapFnt{\normalfont\small}\SetAlCapNameFnt{\unskip\itshape\small}

\definecolor[named]{ACMBlue}{cmyk}{1,0.1,0,0.1}
\definecolor[named]{ACMYellow}{cmyk}{0,0.16,1,0}
\definecolor[named]{ACMOrange}{cmyk}{0,0.42,1,0.01}
\definecolor[named]{ACMRed}{cmyk}{0,0.90,0.86,0}
\definecolor[named]{ACMLightBlue}{cmyk}{0.49,0.01,0,0}
\definecolor[named]{ACMGreen}{cmyk}{0.20,0,1,0.19}
\definecolor[named]{ACMPurple}{cmyk}{0.55,1,0,0.15}
\definecolor[named]{ACMDarkBlue}{cmyk}{1,0.58,0,0.21}

\usepackage[colorlinks,citecolor=blue,linkcolor=magenta,bookmarks=true]{hyperref}
\usepackage[nameinlink]{cleveref}
\crefname{ineq}{Inequality}{Inequality}
\creflabelformat{ineq}{#2{\upshape(#1)}#3}
\crefname{sub}{Subsection}{Subsection}
\creflabelformat{Subsection}{#2{\upshape(#1)}#3}
\crefname{sdp}{SDP}{SDP}
\creflabelformat{sdp}{#2{\upshape(#1)}#3}
\crefname{lp}{LP}{LP}
\creflabelformat{lp}{#2{\upshape(#1)}#3}
\usepackage{aliascnt}
\crefname{ineq}{Inequality}{Inequality}
\creflabelformat{ineq}{#2{\upshape(#1)}#3}
\crefname{sub}{Subsection}{Subsection}
\creflabelformat{Subsection}{#2{\upshape(#1)}#3}
\crefname{sdp}{SDP}{SDP}
\creflabelformat{sdp}{#2{\upshape(#1)}#3}
\crefname{lp}{LP}{LP}
\creflabelformat{lp}{#2{\upshape(#1)}#3}

\makeatletter
\newenvironment{Ualgorithm}[1][htpb]{\def\@algocf@post@ruled{\kern\interspacealgoruled\hrule  height\algoheightrule\kern3pt\relax}\def\@algocf@capt@ruled{under}\setlength\algotitleheightrule{0pt}\SetAlgoCaptionLayout{centerline}\begin{algorithm}[#1]}
{\end{algorithm}}
\makeatother

      \definecolor{mycyan}{RGB}{42,161,152}
\definecolor{myred}{RGB}{220,50,47}
\definecolor{gray}{RGB}{70,70,70}

\newtheorem{theorem}{Theorem}
    \newaliascnt{lemma}{theorem}
    \newtheorem{lemma}[lemma]{Lemma} 
    \aliascntresetthe{lemma}
    \crefname{lemma}{lemma}{lemmas}
        \newaliascnt{proposition}{theorem}
    \newtheorem{proposition}[proposition]{Proposition} 
    \aliascntresetthe{proposition}
    \crefname{proposition}{proposition}{propositions}
     \newaliascnt{corollary}{theorem}
    
    \aliascntresetthe{corollary}
    \crefname{corollary}{corollary}{corollaries}
        \newaliascnt{definition}{theorem}
    \newtheorem{definition}[definition]{Definition}
    \aliascntresetthe{definition}
    \crefname{definition}{definition}{definitions}

\newtheorem{informal}[theorem]{Informal Theorem}

\newaliascnt{claim}{theorem}
  \newtheorem{claim}[claim]{Claim}
    \aliascntresetthe{claim}
    \crefname{claim}{claim}{claims}

  \newaliascnt{fact}{theorem}
\newtheorem{fact}[fact]{Fact}
    \aliascntresetthe{fact}
    \crefname{fact}{fact}{facts}
            \newaliascnt{remark}{theorem}
    \newtheorem{remark}[remark]{Remark}
    \aliascntresetthe{remark}
    \crefname{remark}{remark}{remarks}

\newcommand{\eqdef}{\coloneqq}

\renewcommand\vec[1]{\mathbf{#1}}
\DeclareMathOperator*{\pr}{\mathbf{Pr}}
\DeclareMathOperator*{\E}{\mathbf{E}}

\DeclareMathOperator*{\argmax}{argmax}

\newcommand{\bx}{\mathbf{x}}
\newcommand{\by}{\mathbf{y}}

\newcommand{\bw}{\mathbf{w}}

\newcommand{\err}{\mathrm{err}}

\newcommand{\R}{\mathbb{R}}

\newcommand{\Z}{\mathbb{Z}}

\newcommand{\eps}{\epsilon}

\newcommand{\poly}{\mathrm{poly}}

\newcommand{\sgn}{\mathrm{sign}}
\newcommand{\sign}{\mathrm{sign}}

\newcommand{\D}{D}

\newcommand{\Ind}{\mathds{1}}
\newcommand{\1}{\Ind}

\newcommand{\littlesum}{\mathop{\textstyle \sum}}

\newcommand{\wt}{\widetilde}

\newcommand{\wstar}{\bw^{\ast}}
\newcommand{\x}{\vec x}
\newcommand{\w}{\vec w}

\newcommand{\tth}{^{(t)}}

\newcommand{\citet}{\cite}
\newcommand{\citep}{\cite}

\author{
Ilias Diakonikolas\thanks{Supported by NSF Medium Award CCF-2107079,
		NSF Award CCF-1652862 (CAREER), and
		a DARPA Learning with Less Labels (LwLL) grant.}\\
UW Madison\\
{\tt ilias@cs.wisc.edu}\\
\and
Vasilis Kontonis\thanks{Supported in part by NSF 
Award CCF-2144298 (CAREER).}\\
UW Madison\\
{\tt kontonis@wisc.edu } \\
\and
Christos Tzamos\thanks{Supported by NSF Award CCF-2144298 (CAREER).}\\
UW Madison and NKUA\\
{\tt tzamos@wisc.edu}
\and
Nikos Zarifis\thanks{Supported in part by NSF award 2023239,  NSF Medium Award CCF-2107079, and a DARPA Learning with Less Labels (LwLL) grant.}\\
UW Madison\\
{\tt zarifis@wisc.edu}\\
}

\begin{document}

\maketitle

\begin{abstract}
In online classification, a learner is presented with a sequence of examples 
and aims to predict their labels in an online fashion so as to minimize 
the total number of mistakes. 
In the self-directed variant, the learner knows in advance 
the pool of examples and can adaptively choose the order in which predictions are made. 
Here we study the power of choosing the prediction order and establish 
the first strong separation between worst-order and random-order learning 
for the fundamental task of linear classification. 
Prior to our work, such a separation was known only for very restricted concept classes, 
e.g., one-dimensional thresholds or axis-aligned rectangles.

We present two main results.
If $X$ is a dataset of $n$ points drawn uniformly at random from the $d$-dimensional unit sphere, 
we design an efficient self-directed learner that
makes $O(d \log \log(n))$ mistakes and classifies the entire dataset.
If $X$ is an arbitrary $d$-dimensional dataset of size $n$, 
we design an efficient self-directed learner that predicts the labels 
of $99\%$ of the points in $X$ with mistake bound independent of $n$. 
In contrast, under a worst- or random-ordering, the number of mistakes 
must be at least $\Omega(d \log n)$, even when the 
points are drawn uniformly from the unit sphere 
and the learner only needs to predict the labels for $1\%$ of them. 
\end{abstract}

\setcounter{page}{0}

\thispagestyle{empty}

\newpage

\section{Introduction}

Online prediction has a rich history going back to the
pioneering works of  \cite{robbins1951asymptotically,
hannan1957approximation,blackwell1954controlled}. In the online setting, 
the learner aims to resolve a
prediction task by learning a hypothesis from a sequence of examples one at a
time. The goal is to minimize the total number of incorrect predictions (aka mistake bound) 
given the knowledge of the correct answers to previously queried examples
\citep{Littlestone:88,Littlestone:89,Blum:90,LittlestoneWarmuth:94,MT:94}.  A
standard, worst-case assumption is that an adversary controls the sequence of
examples and/or labels.  In this worst-case setting, a wide range of algorithms based on
exponential reweighting
\citep{vovk1990aggregating,LittlestoneWarmuth:94,FreundSchapire:97,
vovk1995game,cesa2006prediction} and online convex optimization
\citep{hazan2016introduction,orabona2019modern} have been developed. 
Motivated by the fact that in many applications the sequence of examples is not
adversarial, the problem of online prediction has also been studied 
in more benign settings, such as assuming that the examples 
are given to the learner by a teacher (who knows the ground-truth hypothesis)
\citep{goldman1993teaching,mathias1997model,doliwa2010recursive,mansouri2022batch}
or making regularity assumptions about the sequence of examples
\citep{rakhlin2013online,jadbabaie2015online}.  In this work, we study
the model of Self-Directed learning \citep{goldman1993learning,goldman1994power}, where the
learning algorithm can choose which example to label next.  Self-directed
learning has many applications. For example, in direct marketing
\citep{ni2011direct}, the learner must study customers' characteristics and
needs and adaptively select examples (customers) to market their products. 
Moreover, it is related to curriculum learning 
--- proposed in the influential work of \cite{bengio2009curriculum} --- 
where the training examples are sorted from ``easier'' to ``harder'' in order to help
the model learn faster (see \Cref{sec:related} for more details).

Here we consider the mistake-bound model \citep{Littlestone:88, Littlestone:89} in
the realizable setting, where the labels revealed to the learner in each round
are consistent with a ground-truth classifier $f$ that belongs to a known 
concept class $\mathcal C$ fixed before the online learning phase starts.  We
now formally define the self-directed online prediction model
\citep{goldman1994power} and its random- and worst-order variants that we
consider in this work.

\begin{definition}[Self-Directed Online Learning \citep{goldman1994power}]
    \label{def:self-directed-learning}
Let $f \in \mathcal{C}$ be an unknown target concept from some concept class
$\mathcal{C}$ of boolean functions from $\R^d$ to $\{\pm 1\}$ and let $X =
\{\x^{(1)},\ldots, \x^{(n)}\}$ be a subset of $n \in \mathbb N$ points in 
$\R^d$.
The learner has access to the full set of (unlabeled) points $X$. 
\\
Until the labels of all examples of $X$ have been predicted: 
\vspace{-0.3em}
\begin{itemize} 
\itemsep0em
    \item The learning algorithm picks a point $\x \in X $ making a prediction
$z \in \{\pm 1\}$ about its label.
    \item The true label $f(\x)$ of $\x$ is revealed to the learning algorithm.
\end{itemize}
\vspace{-0.3em}
We say that the learner makes $M$ mistakes to label $X$ if, with probability at
least $99\%$, it holds that the number of incorrect predictions of the learner
is at most $M$.
\end{definition}
In this work, we investigate the power of the above self-directed online
learning setting compared to the worst- and random-order settings.
\begin{remark}[Random-order and Worst-order Online Learning]
	We shall refer to the setting where the point $\x$ during the training phase is
	picked uniformly at random (without replacement) from the unlabeled data $X$ as
	\textbf{random-order} learning.  Moreover, we  shall refer to the setting where
	the next example $\x$ is chosen by an adversary as \textbf{worst-order}
	learning. 
\end{remark}

Observe that in \Cref{def:self-directed-learning} the learner must predict
labels for \emph{all examples} of the dataset $X$.  We will also consider weaker
learning notions allowing the learner to avoid labeling a fraction of examples
in $X$.  As we will see, these weaker learning notions are especially useful
when dealing with an unstructured dataset $X$, containing potentially
ambiguous or adversarial examples that may be inherently hard to predict.
\begin{remark}[Perfect, Strong, and Weak Self-Directed Learners]
{\em	We will refer to the setting of \Cref{def:self-directed-learning}, where the
	learner predicts the labels of all examples in $X$, as \textbf{perfect}
	self-directed learning.  Moreover, we refer to the setting where the learning
	algorithm provides labels for at least $(1-\epsilon)$-fraction of $X$ for every
	$\epsilon \in (0,1]$ as \textbf{strong} self-directed learning. Finally, in
	\textbf{weak} self-directed learning, the learner predicts the labels of some
	fixed fraction of $X$, say $1\%$.}
\end{remark}

\paragraph{Online Linear Classification}
We focus on the fundamental setting of online linear
classification that dates back to Rosenblatt's perceptron
~\citep{Rosenblatt:58}.  A (homogeneous) {\em halfspace} or Linear Threshold
Function (LTF) is a Boolean-valued function $f: \R^d \mapsto \{\pm 1\}$ of the
form $f(\x) = \sgn(\vec w^\ast \cdot \x)$, for a vector $\vec w^\ast \in \R^d$
(known as the weight vector).  Halfspaces are a central class of Boolean
functions in several areas of computer science, including complexity theory,
learning theory, and optimization~\cite{Novikoff:62, Yao:90, GHR:92,
FreundSchapire:97, Vapnik:98, CristianiniShaweTaylor:00,}.  In (realizable)
online linear classification, the halving algorithm --- first appeared in
\cite{barzdi1972prediction} and further analyzed and refined in
\cite{Mitchell:82,Angluin:87,Littlestone:88} --- makes $O(d \log n)$ mistakes for
perfect classification of $n$ points in $d$ dimensions.  This mistake bound and
in particular its dependence on the dataset size $n$ is known to be
information-theoretically optimal for linear classification \citep{Littlestone:88} 
in the worst-order online learning setting.  
We show  (see \Cref{pro:random-order-lower-bound})
that, even when the order of examples is random, $\Omega(d \log n)$ mistakes
are required.  In fact, this is true even when the $n$ points of $X$ are drawn
uniformly from the $d$-dimensional unit-sphere $\mathbb{S}_d$ and the learner
only needs to predict the labels of $1\%$ of them (weak-learning).

\paragraph{Improved Mistake Bounds via Self-Directed Learning}
In the self-directed learning setting, it is known \citep{goldman1994power} that
for one-dimensional threshold functions one mistake suffices --- significantly
improving over the $\Theta(\log n)$ mistake-bound for the same class in the
worst-order setting \citep{Littlestone:88,Littlestone:89}. Similar improvements
in self-directed learning have also been shown for other simple concept classes
such as monotone-monomials and axis-aligned rectangles (see
\citet{ben1997online} for a discussion on the gaps between self-directed
learning and worst-order learning).  In this work, we study whether
self-directed learning can improve the number of mistakes to learn more
complicated, high-dimensional concept classes focusing on the fundamental class
of $d$-dimensional LTFs.  We remark that, beyond the $O(d \log n)$
mistake-bound given by the halving algorithm, no other mistake bounds are known
for self-directed learning of halfspaces, even when the dataset $X$ is assumed
to be structured, e.g., the $n$ examples of $X$ are drawn uniformly at random
from the $d$-dimensional unit sphere $\mathbb S_d$.  We ask the following
natural question.
\begin{center}
\emph{Can self-directed learners bypass the $\Omega(d \log n)$ mistake barrier 
of worst- and random-order learning for $d$-dimensional LTFs?}
\end{center}
We give a positive answer to the above question showing that allowing the
learner to choose the order of examples significantly improves the mistake bound 
under both structured and arbitrary datasets.

\subsection{Our Results and Techniques}

Our first result assumes that the dataset $X$ is structured: its $n$ examples 
are drawn uniformly at random from the unit sphere $\mathbb S_d$.  In this case, 
we give a self-directed learning algorithm that only makes $O(d \log \log n)$ mistakes 
to predict the labels of all examples in $X$ (\textbf{perfect learning}), establishing
an exponential improvement over the $\Omega(d \log n)$ mistake bound 
of worst- and random-order learners. 
Since it is known \citep{ben1997online} that even self-directed learners must make
$\Omega(d)$-mistakes, our $O(d \log \log n)$ is close to best-possible.

\begin{informal}[Perfect Self-Directed Learner on $\mathbb{S}_d$]\label{inf-thm:sd-learning-sphere}
There exists a perfect self-directed learner for halfspaces that makes 
$O(d \log \log n)$ mistakes to classify a set of $n$ points 
drawn uniformly from $\mathbb{S}_d$.
\end{informal}

Our learner maintains a 
halfspace hypothesis and at each round predicts the label of the example with the largest margin from the current halfspace.  At a high level, this corresponds to
choosing the example for which the current hypothesis is most-confident in an 
``easy examples first'' manner.   When the prediction on such an example is 
incorrect, we use the so-called margin-perceptron update rule used in 
\cite{DunaganV04} in the context of linear programming.  We show that the margin-perceptron update used on examples for which the classifier is very confident (large-margin) -- but made a mistake -- converges super-linearly to the ground-truth 
halfspace $\vec w^\ast$.
For a more detailed overview, we refer to \Cref{sec:spherical-roadmap}.

Our second result is a self-directed learner for arbitrary datasets.  In this case,
we give a \textbf{strong} self-directed learner that makes $\poly(d)$ mistakes
and labels $99\%$ of $X$. Recall that, in contrast, any worst- or random-order
learner makes $\Omega(d \log n)$ mistakes to label $1\%$ of the dataset (even when
the data are drawn uniformly from the unit sphere).

 \begin{informal}[Strong Self-Directed Learner for Arbitrary Data]\label{inf-thm:sd-distribution-free}
There exists a strong, self-directed learner for halfspaces that, given an arbitrary set $X$ of $n$ points in $d$ dimensions, makes 
    $\poly(d)$ mistakes to classify $99\%$ of $X$.
\end{informal}

At a high level, similarly to our algorithm for uniformly spherical data, 
our self-directed learner for arbitrary data again picks
the examples with the largest margin from the current hypothesis and 
uses the margin-perceptron update rule of \cite{DunaganV04}.
The second ingredient of our learner is the Forster transform \citep{forster2002} 
that puts the examples in (approximately) Radially Isotropic Position,
i.e., make $X$ isotropic and also normalize all points so that they lie on the unit sphere; see \Cref{def:rip}.  Recent works (\citet{HardtM13,AKS20,DTK22})
have provided efficient algorithms to compute the Forster transform.
When the dataset is in Radially Isotropic Position, one can show that it satisfies
a ``soft-margin'' condition (i.e., that a non-trivial fraction of $X$ has non-trivial
margin with every halfspace).  We use this property to first obtain a 
\emph{weak learner} that does $O(d \log d)$ mistakes to label an 
$\Omega(1/d)$-fraction of $X$.  We then use a generic boosting approach to transform
this weak learner into a \emph{strong learner} that does $\poly(d)$ mistakes to label
$99\%$ of $X$, see \Cref{sec:roadmap-distribution-free-sd} and \Cref{ssec:boosting}.

\subsection{Related Work}
\label{sec:related}
Related to the setting of self-directed learning is active learning \citep{cohn1994improving},
where the learner has access to a large pool of unlabeled examples and chooses the ``most informative'' to ask for their labels. 
The goal is to find a classifier with good generalization while minimizing the number of label queries. 
There is a long line of research on active linear classification in the distribution-specific setting (e.g., under the uniform distribution on the unit sphere) \citep{dasgupta2005,Hanneke2011,balcan2016active}.  We remark that our goal
of minimizing the number of mistakes is orthogonal to that of active learning:
at a high-level, our algorithms pick the examples for which the current hypothesis
is most confident (``easiest examples'') while in active learning one typically asks for the labels of the ``hardest examples'',  e.g., those with the smallest margin
with respect to the current guess (see, e.g., \citep{AwasthiBHU15,AwasthiBHZ16,ZSA20}).

In deep learning, stochastic gradient descent typically trains models by considering
the examples in a random order.  In the influential work of \cite{bengio2009curriculum}
the authors proposed curriculum learning: training machine learning models in a 
``meaningful order'' -- from easy examples to harder ones.  There is a long line 
of research (see the surveys \citep{hacohen2019power,wang2021survey,soviany2022curriculum} and references therein) giving empirical evidence that curriculum learning provides significant
benefits in convergence speed and generalization over training with random order.
Our results provide theoretical evidence that ordering the examples 
from easier to harder significantly reduces the mistakes made by the learner.

\section{Self-Directed Learning on $\mathbb{S}_d$}\label{sec:uniform}

In this section, we present our self-directed learning algorithm for datasets uniformly distributed on the  unit sphere.  We first state the formal version of \Cref{inf-thm:sd-learning-sphere}.
\begin{theorem}[Perfect Self-Directed Learner on $\mathbb S_d$]
\label{thm:sd-learning-sphere}
Let $\delta \in (0,1/2]$ and let $n$ be larger than 
some sufficiently large universal constant.
Let $X$ be a set of $n$ i.i.d.\ samples from $\mathbb S_d$
with true labels given by a homogeneous halfspace 
$f(\x) = \sgn(\vec w^\ast \cdot \x)$.
There exists a self-directed classifier that makes $O(d \log \log n ~  \log(1/\delta))$ mistakes, runs in time $\poly(d,n)$ and classifies 
all points of $X$ with probability at least $1-\delta$.
\end{theorem}

\subsection{Roadmap of the Proof of \Cref{thm:sd-learning-sphere}}
\label{sec:spherical-roadmap}

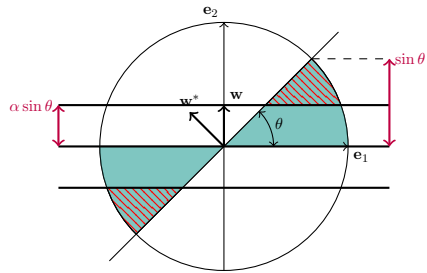
\begin{wrapfigure}[16]{r}{0.4\textwidth}
	\centering
	\begin{tikzpicture}[scale=0.55, every node/.style={transform shape} ]
	\coordinate (start) at (0.5,0.5);
	\coordinate (center) at (0,0);
	\coordinate (end) at (0,0.7);
	\coordinate (end2) at (-0.12/0.7634,0.366/0.7634);
	\coordinate (start2) at (0,0.7);
	\coordinate (start3) at (3,0);
	\coordinate (end3) at (45:3);
	\coordinate (start4) at (135:2);
	\coordinate (end4) at (-2,0);
	\draw (0,0) circle (3);

\draw[fill=mycyan, opacity=0.6] (0,0) -- (3,0) arc (0:45:3.0cm) -- cycle;
 \begin{scope}[thin]
\clip(0,1) rectangle (3,3);
\draw[pattern=north west lines, pattern color=red, opacity=1] (0,0) -- (3,0) arc (0:45:3.0cm) -- cycle;
\end{scope}

\draw[fill=mycyan, opacity=0.6] (0,0) -- (-3,0) arc (180:225:3.0cm) -- cycle;
	 \begin{scope}[thin]
\clip(0,-1) rectangle (-3,-3);
\draw[pattern=north west lines, pattern color=red, opacity=1] (0,0) -- (-3,0) arc (180:225:3.0cm) -- cycle;
\end{scope}
	\draw[black,thick,-] (-4,0) -- (4,0) ;
 	\draw[black,thick,-] (-4,1) -- (4,1) ;
  	\draw[black,thick,-] (-4,-1) -- (4,-1) ;
	\draw[->] (-3,0) -- (3,0) node[anchor=north west,black] {$\vec e_1$};
	\draw[->] (0,-3) -- (0,3) node[anchor=south east] {$\vec e_2$};
\draw[black,thick,->] (0,-0) -- (-0.707*4.7/4,0.707*4.7/4) node[anchor= south ] {$\wstar$};
\draw[purple,thick,<->] (-4,0) -- (-4,1)node[anchor= south east,left] {$\alpha \sin\theta$};
 \draw[purple,thick,<->] (4,0) -- (4,2.12)node[anchor= south east,right] {$ \sin\theta$};
 
	\draw[black] (-0.707*4.7/1.2,-0.707*4.7/1.2) -- (0.707*4.7/1.2,0.707*4.7/1.2);
	\draw[black,thick ,->] (0,0) -- (0,1) node[above right] {$\bw$};
\pic [draw, <->, angle radius=12mm, angle eccentricity=1.2, "$\theta$"] {angle = start3--center--end3};
\draw[black,dashed ,-] (2*1.414*3/4,2*1.414*3/4) -- (4,2.12) node[above right] {};
\end{tikzpicture}
	\caption{The probability that the maximum-margin mistake is has
 margin larger than $\alpha \sin \theta$ corresponds to the probability
 of the shaded subset of the disagreement region (shown in cyan).}
	\label{fig:uniform-large-margin}
\end{wrapfigure} The first ingredient of our algorithm 
is an adaptive way to pick examples: at every step the 
learner predicts the labels of examples for which the current hypothesis is
most confident.  The intuition behind this choice is that for 
those examples the hypothesis is more often correct and,
when it predicts incorrectly, they can be used to improve
it significantly.  The second ingredient is the ``margin-perceptron''
algorithm of \cite{DunaganV04}.  The margin-perceptron iteration
is a variant of the standard perceptron update rule that scales
the update with the signed margin of the example, i.e.,
given an example $\x$ that the current hypothesis $\vec w^{(t)}$
predicts incorrectly (i.e., $\sgn(\vec w^{(t)} \cdot \x) 
\neq \sgn(\vec w^\ast \cdot \x)$), we update $\vec w^{(t)}$ as
follows:
\begin{equation}
\label{eq:margin-perceptron}
\vec w^{(t+1)} \gets \vec w^{(t)} - (\vec w^{(t)} \cdot \x)  \x 
\end{equation}
Contrary to the standard perceptron update (i.e., 
$\vec w^{(t+1)} \gets \vec w^{(t)} - \x $), the perceptron update rule of \Cref{eq:margin-perceptron} does 
not rely on improving the correlation with the target vector 
$\vec w^\ast$ but on \emph{decreasing the norm} $\|\vec w^{(t)}\|_2$
and \emph{not decreasing the correlation} with $\vec w^\ast$.
We show that, as long as the margin  $|\vec w^{(t)} \cdot \x|$ 
is large, the margin-perceptron update of \Cref{eq:margin-perceptron} will 
significantly reduce the angle $\theta(\vec w^{(t)}, \vec w^\ast)$ 
between $\vec w^{(t)}$ and $\vec w^\ast$.

\paragraph{Reducing $\tan (\theta(\vec w^{(t)}, \vec w^{\ast}))$ via Margin-Perceptron}

We first observe that given an example that $\vec w^{(t)}$ mispredicts, 
its maximum
possible margin with the current guess $\vec w^{(t)}$ is equal to
$\sin(\theta(\vec w^{(t)}, \vec w^\ast))$; see \Cref{fig:uniform-large-margin}.
We show (see \Cref{lem:update-no-buckets-main}) 
that when 
$|\vec w^{(t)} \cdot \x| \geq r \sin \theta(\vec w^{(t)}, \vec w^\ast))$, 
the margin-perceptron update reduces $\tan(\theta(\vec w^{(t)} , \vec w^\ast))$ multiplicatively:
\begin{equation}
\label{eq:tantheta-decay}
\tan^2(\theta(\vec w^{(t+1)}, \vec w^\ast))
\leq  (1- r^2) \tan^2(\theta(\vec w^{(t)}, \vec w^\ast)) 
\,.
\end{equation}
Observe that, the closer $r$ is to $1$ (i.e., the larger the margin), 
the faster $\theta^{(t)} = \theta(\vec w^{(t)}, \vec w^{\ast})$ will
converge to $0$.  Using the fact that the distribution is uniform on 
the sphere, the probability that some halfspace with normal vector 
$\vec u$ disagrees with the  ground-truth halfspace $\vec w^\ast$ is equal to  $\theta(\vec u, \vec w^\ast)/\pi$.  Thus, achieving 
$\theta(\vec w^{(t)}, \vec w^\ast) \leq O(1/n)$, implies
that, in expectation, the number of mistakes on a sequence of $n$ i.i.d.\
examples from $\mathbb S_d$ is going to be $O(1)$; see \Cref{lem:generalization}.
Therefore, in what follows, our goal will be to make $\theta^{(t)}$
smaller than $O(1/n)$. \footnote{More precisely, it suffices to have $\theta^{(t)} \leq O(d \log \log n /n)$, as this would imply an $O(d \log \log n)$ mistake
bound on a sequence of $n$ data.}

\paragraph{Achieving Large Margin on $\mathbb S_d$}
Given the current hypothesis with normal vector $\vec w^{(t)}$, 
a uniformly random point on the $d$-dimensional sphere has
margin roughly $\Omega(1/\sqrt{d})$.  Given that the point falls in
the disagreement region of $\vec w^{(t)}$ and $\vec w^\ast$, 
the margin can be shown to be roughly 
$\Omega(\sin \theta^{(t)}/\sqrt{d})$. Therefore, in each round,
\Cref{eq:tantheta-decay} implies that the angle $\theta$ (recall that for small $\theta$
it holds that $\theta \approx \tan \theta$) is going to decrease roughly 
by a factor of $\sqrt{1-\Omega(1/d)}$:
\[
\theta^{(t+1)} \leq  \sqrt{1- \Omega(1/d)} ~ \theta^{(t)}
\,.
\]
In order to make $\theta^{(t)} \leq O(1/n)$, 
the above iteration requires roughly $d \log n$ updates which 
only implies a mistake bound of  $O(d \log n)$.
This is where the fact that at every iteration $t$ we choose the example with maximum margin with respect to $\vec w^{(t)}$ comes into play.
We show that for $n$ examples distributed uniformly on the sphere, 
the maximum-margin mistake with respect to $\vec w^{(t)}$
has margin roughly 
$\Omega(\sqrt{\log(n \theta^{(t)})/ d}) ~ \sin(\theta^{(t)})$ 
when $n \theta^{(t)} \leq e^{O(d)}$, 
and 
$(1- (1/(n \theta^{(t)}))^{2/d}) ~ \sin(\theta^{(t)}) $ when $n \theta^{(t)}$ is larger
than $e^{O(d)}$.  We show the following proposition; for the formal
statements, see \Cref{pro:conditional-max-margin} and 
\Cref{cor:unconditional-max-margin}.
\begin{proposition}[Informal: Max-Margin in the Disagreement Region]
\label{pro:informal-unconditional-max-margin}
Let $\vec v, \vec u \in \R^d$ be unit vectors with angle $\theta(\vec v, \vec u) = \theta$.
Let $C$ be the indicator of the disagreement region of the two homogeneous
halfspaces defined by  $\vec v, \vec u$, i.e., $C = \1\{ (\vec v \cdot \x) 
(\vec u \cdot \x) \leq 0\}$.
Let $X$ be a dataset with $n$ i.i.d.\ samples from $\mathbb S_d$.
Then with probability at least $2/3$:
\vspace{-0.4em}
\begin{enumerate}
\itemsep0em
\item 
    If $n \theta \leq e^d$, it holds that
    \(
    \max_{\x \in X \cap C} |\vec u \cdot \x| \geq 
    \Omega(\sqrt{ {\log (n \theta) }/{d} })  
    \sin \theta
    \).
\item  
    Otherwise, 
    \(
    \max_{\x \in X \cap C} |\vec u \cdot \x|
    \geq (1- ({1}/{(n \theta)})^{2/d}) \sin\theta 
    \).
\end{enumerate}
\end{proposition}
We observe that when $n \theta \to \infty$ the maximum-margin 
over the $n$ samples converges to its maximum value of $\sin \theta$.
Since the analysis of the general case turns out to be similar
to the case of $n \theta \leq e^d$,  for simplicity, in this 
overview we will focus on this case.  

\paragraph{(Stochastic) Super-Linear Convergence}
We observe that by taking
the maximum-margin sample and using \Cref{pro:informal-unconditional-max-margin}, we improved the decay of the angle to roughly:
\begin{equation}
\label{eq:super-linear-angle-decay}
\theta^{(t+1)} \leq  \sqrt{1- \frac{\log(n\theta^{(t)}) }{d} } ~ \theta^{(t)}
\leq 
e^{-\log(n \theta^{(t)})/(2d) } ~ \theta^{(t)} 
=
 (\theta^{(t)})^{1-1/(2d)} (1/n)^{1/(2d)} \,.
\end{equation}
Therefore, the angle after an update on the maximum-margin mistake 
is \emph{the weighted geometric mean between  $\theta^{(t)}$ and $1/n$}
with weights $1-1/(2d)$ and $1/(2d)$.  It is not hard to show that 
after $m = O(d \log \log n)$ such updates we have $\theta^{(m)}
\leq O(1/n)$, achieving our goal.  One issue is that \Cref{pro:informal-unconditional-max-margin} only gives ``good'' probability, i.e., 2/3, 
that such a large-margin update will happen.  In \Cref{lem:super-linear-convergence-main}, we show that increasing the iterations by a constant factor is enough to show that $\theta^{(m)}$ will be $O(1/n)$ with good probability.

\paragraph{Dealing with the Dependencies}
In our discussion so far, we have ignored the fact that
if at some step $t$ the algorithm searches over the whole dataset $X$
in order to find the example $\x$ with the largest margin 
with respect to the guess $\vec w^{(t)}$, in the next step the remaining points
of $X$ are no-longer i.i.d. samples from $\mathbb S_d$; therefore,
many of our claims using the independence of the samples (e.g.,
\Cref{pro:informal-unconditional-max-margin}) no longer work.
We handle this issue by only considering a large-enough subset of examples
in each step of the margin-perceptron update, i.e., instead of selecting
the point of maximum-margin over the whole dataset, we select 
the maximum-margin example of a random subset (see Steps 3,4 in 
\Cref{alg:sd-margin-perceptron-sphere-main}) 
inside which we only perform a single
margin-perceptron update.  Since the total number of margin-perceptron
updates required is only $O(d \log \log n)$, we split the dataset
into $k = O(d \log \log n)$ random subsets of equal size; 
therefore, assuming that $n$ is larger than roughly $\Omega(d \log \log d)$, 
each bucket will have enough samples to guarantee that a margin-perceptron update with large margin will happen with good probability.

Finally, as we use examples of $X$ to update the guess $\w^{(t)}$,
when we reach the target angle, say $\theta^{(t)} \leq O(1/n)$, we cannot
guarantee that $\vec w^{(t)}$ will make few mistakes on the samples that we
used to train it (as $\vec w^{(t)}$ depends on those samples).
To avoid this issue, we split the initial dataset into two random subsets
of equal size $A$ and $B$, and then train a linear classifier for each part.
In the final step, we use the linear classifier trained on $A$ to label the dataset $B$ (that was not used during its training) and the linear classifier trained on $B$ to label the examples of $A$; see 
Step 5 in \Cref{alg:sd-margin-perceptron-sphere-main}.

\begin{Ualgorithm}
	\centering
	\fbox{\parbox{5.8in}{
 {\bf Input:} An initialization $\w$. 
 {\bf Output:} A sequence of labeled data $(\x^{(t)}, z^{(t)})$.
\begin{enumerate}
  \setlength\itemsep{0.1em}
\item Initialize guesses $\vec w^{(0)} \gets \vec w$,
~~~
$\vec v^{(0)} \gets \vec w$.
\item Initialize the set of unlabeled data $U \gets X$.
\item Split $U$ in $2 k$ sets $U_{1},\ldots, U_{2k}$.

\item For $t =1,\ldots, k$:
\smallskip
\\
$~~\vec w^{(t)} \gets$ \textsc{Margin-Perceptron}($U_t$, $\vec w^{(t-1)}$), 
$\vec v^{(t)} \gets$ \textsc{Margin-Perceptron}($U_{k + t}$, $\vec v^{(t-1)}$).
\item For $t =1,\ldots, k$:
\\
$~~$ Label points of $U_{k + t}$ with $\vec w^{(k)}$
    and  label points of $U_t$ with $\vec v^{(k)}$.
\end{enumerate}
\centering
	\fbox{\parbox{5.5in}{
\textsc{Margin-Perceptron}($U,\vec w$)\\
 {\bf Input:} An initialization $\w$ and a set of points $U$. 
 {\bf Output:} A vector $\vec w'$.
\begin{enumerate}
  \setlength\itemsep{0.1em}
\item Obtain $U'$ by sorting the points of $U$ in decreasing order of margin from 
        $\vec w$, i.e., $|\vec x^{(i+1)} \cdot \vec w| \leq
        |\vec x^{(i)} \cdot \vec w|$.
\item For $\x\in U'$:
\begin{enumerate}
  \setlength\itemsep{0.1em}
\item Predict the label of $\x$ with $\vec w$.
\item If the prediction is incorrect, 
exit the loop and return $\vec w' \gets \vec w-(\vec w\cdot \x)\x$.
\end{enumerate}
\end{enumerate}}}
    }}
    \vspace{0.1cm}
	\caption{Self-Directed Learning on $\mathbb S_d$}
	\label{alg:sd-margin-perceptron-sphere-main}
\end{Ualgorithm}

\subsection{Proof of \Cref{thm:sd-learning-sphere}}

We first give a proof sketch of \Cref{pro:informal-unconditional-max-margin}
showing that maximum-margin mistakes will have margin significantly larger
than the margin of an ``average'' mistake.  In the following sketch
we only show the first case of \Cref{pro:informal-unconditional-max-margin};
for the full proof we refer to \Cref{app:uniform}.

\smallskip

\begin{proof}{(sketch of \Cref{pro:informal-unconditional-max-margin})}
We observe that by the rotational symmetry of the uniform distribution on the sphere, the probability that a sample $\x$ falls in the disagreement region $C$ is exactly  $\theta/\pi$; see \Cref{fig:uniform-large-margin}.
Therefore, on expectation, out of the $n$ samples that
we draw from $\mathbb S_d$, $n \theta/\pi$ fall in $C$.
Since it is not hard to show that with
high probability $\Omega(n\theta)$ samples fall in $C$ (see \Cref{app:uniform});
for this sketch we assume that this is case with probability $1$.

We now show that conditionally on 
observing $m$ samples in the disagreement region $C$,  the maximum-margin has strong \emph{anti-concentration}.
Denote by $\mathbb S_d(C)$ the conditional distribution on the disagreement region $C$.
For any $\alpha \in [0, 1]$, it holds:
\[
\pr_{\x^{(1)},\ldots, \x^{(m)} \sim \mathbb S_d(C)} 
\left[\max_{i=1,\ldots,m} |\vec u \cdot \x^{(i)}| \leq 
\alpha\sin(\theta/2)
\right] \leq \exp\left(-m ~  (1-\alpha^2)^{d/2-1}/2\right)  \,.
\]
To simplify notation, set $\Delta = \alpha \sin (\theta/2)$.
We first compute the probability that a single sample in $C$ 
has $|\vec u \cdot \x| \geq \Delta$.
By the symmetry of the set $C$ and the uniform distribution on the sphere, 
it holds 
\(
\pr_{\x \sim \mathbb S_d}[ \vec x \in C, |\vec u \cdot \x| \geq \Delta]
=
2 
\pr_{\x \sim \mathbb S_d}[ \vec x \in C, \vec u \cdot \x \geq \Delta] 
=
2 
\pr_{\x \sim \mathbb S_d}[E_1]\, ,
\)
where $E_1 = \{\x : \vec x \in C, \vec u \cdot \x \geq \Delta\}$
($E_1$ corresponds to the upper shaded cell in \Cref{fig:uniform-large-margin}).
Assume without loss of generality that $\vec u = \vec e_2$
and $\vec v = - \sin \theta \vec e_1 + \cos \theta \vec e_2$.
Using polar coordinates $\x_1 = r \cos \phi$,
$\x_2 = r \sin \phi$, we have that (see \Cref{fig:uniform-large-margin} and \Cref{app:uniform}),
\(
E_1 = \{(r,\phi): \alpha \leq r \leq 1, r \sin \phi 
\geq \alpha \sin \theta, \phi \leq \theta\}
\).
The set $E_1$ has coupled constraints 
(i.e., constraints that depend on both $r, \phi$).
To avoid this, we show that the set 
$E_2 = \{(r, \phi): \alpha \leq r \leq 1, 
\theta/2 \leq \phi \leq \theta  \} $ that has decoupled constraints is a subset of $E_1$.  
The $2$-dimensional marginal of the uniform distribution on $\mathbb S_d$
has density $\frac{d-2}{2 \pi} (1-r^2)^{d/2-2} r$ (in polar coordinates).  
\begin{align*}
\pr_{\x \sim \mathbb S_d}[E_2]
=\frac{d-2}{2 \pi} \int_{\alpha}^1 \int_{\theta/2}^\theta 
(1-r^2)^{d/2-2} ~ r ~ d \phi d r
=
 \frac{\theta}{4 \pi} (1-\alpha^2)^{d/2-1}
 \,.
\end{align*}
Recall that by, the symmetry of $\mathbb S_d$, we directly obtain that
$\pr_{\x \sim \mathbb S_d}[C] = \theta /\pi$.
We conclude that the conditional probability 
$\pr_{\x \sim \mathbb S_d(C)}[ |\vec u\cdot \x| \geq \Delta]
\geq (1/2) (1-\alpha^2/d)^{d/2-1}$.
We can now bound by above the probability that the maximum-margin 
of $m$ independent samples from $\mathbb S_d(C)$ is small.
\begin{align*}
\pr_{\x_1,\ldots, \x_m \sim \mathbb S_d(C)} 
&\left[\max_{i=1,\ldots,m} |\vec u \cdot \x_i| \leq \Delta
\right] 
= (1- \pr_{\x \sim \mathbb S_d(C)}[|\vec u \cdot \x| \geq \Delta])^m
\\
&\leq \exp\left(-m \pr_{\x \sim \mathbb S_d(C)}[|\vec u \cdot \x| \geq \Delta]\right)
\leq  \exp\left(-m ~ (1-\alpha^2)^{d/2-1}/2\right) \,,
\end{align*}    
where, for the first inequality, we used the fact $e^{x} \geq 1 + x$.
Using $m = \Omega(n \theta)$ (since we know that roughly $n\theta$ examples land in the disagreement region $C$) and $\alpha = \Omega(\sqrt{\log m}/{\sqrt{d}})$, 
we obtain the result (for the first case of \Cref{pro:informal-unconditional-max-margin}). 
\end{proof}

We prove the following lemma, showing on each mistake
the margin-perceptron has good probability of significantly 
(super-linearly) decreasing $\tan \theta$.  We require that 
the current guess $\vec w$ is not exactly orthogonal with the target
$\vec w^\ast$ (notice the assumption $1/\cos \theta \leq O(\zeta)$).
We show that it is not hard to obtain such an initialization. 
\begin{lemma}[Stochastic Multiplicative Decay of $\tan \theta$]
\label{lem:update-no-buckets-main}
  Let $\vec w \in \R^d$ and let $\theta(\vec w^\ast,\vec w)=\theta \in [0, \pi/2)$ and assume that for some $\zeta>0$, 
  $1/\cos\theta \leq \zeta/(12 \pi)$.
  Let $C$ be the indicator of the disagreement region of $\vec w^\ast,\vec w$, i.e., $C = \1\{ (\vec w^\ast \cdot \x) 
(\vec w \cdot \x) \leq 0\}$. Let $ X=\{\x^{(1)},\ldots,\x^{(n)}\}$ be a sample set drawn from $\mathbb{S}_d$ with $n$ larger than a sufficiently large constant and let $\widehat{\x}=\argmax_{\x\in X}  |\vec w \cdot \x| \1\{ \x \in C\}$. Let $\vec w' =\vec w -
(\vec w\cdot \widehat{\x})\widehat{\x}$ and 
$\theta'=\theta(\vec w',\vec w^\ast)$. 
\begin{enumerate}
\item  We have that $\tan^2(\theta')\leq \tan^2(\theta)$. (Monotonicity)
\item  With probability at least $2/3$, it holds 
\(
\tan(\theta')\leq \tan^{1-1/(8d)}(\theta) ({\zeta}/{n})^{1/(8d)}\;.
\)
\end{enumerate}
\end{lemma}
\begin{proof}{(sketch)}
First, we assume that we perform an update on an example where
    $(\x\cdot \vec w) (\x\cdot \vec w^\ast) < 0$ (i.e., $\vec w$ makes a mistake on $\x$) and the margin of $\x$ is large: $|\vec w \cdot \x|\geq r\sin\theta$. Assuming that $\theta \in [0, \pi/2)$ we can show that $\tan \theta$ decreases multiplicatively.
We first observe that the correlation with $\vec w^\ast$ does not decrease,
\(
            \w'\cdot\wstar=(\w-(\w \cdot\x)\x)\cdot \wstar =\w\cdot\wstar -(\w
    \cdot\x)(\x\cdot \wstar )\geq \w\cdot\wstar
    \), where we used that the hypothesis $\vec w$ disagrees with the ground-truth
      $\vec w^\ast$ on $\x$, i.e., $(\w \cdot\x)(\x\cdot \vec w^\ast)\leq 0$. 
    Furthermore, since $\vec w\cdot \vec w^\ast \geq 0$ (by the assumption that 
    $\theta \in [0,\pi/2)$) we also have that 
    $(\vec w'\cdot \wstar)^2 \geq (\vec w\cdot \wstar)^2$.
We next show that the norm of $\w'$ decreases multiplicatively.
    We have that
\(
       \|\w'\|_2^2= \|(\w-(\w \cdot\x)\x)\|_2^2=\|\w\|_2^2-(\w \cdot\x)^2
        \leq \|\w\|_2^2(1-r^2\sin^2\theta)
        \,.
        \)
Using (twice) the trigonometric identity $\tan^2\phi =(1/\cos^2\phi)-1$,  
    we show that $\tan^2\theta$ decreases by a factor of $(1-r^2)$:
    \begin{align*}
        \tan^2\theta'&=\frac{\|\vec w'\|_2^2}{(\vec w'\cdot \wstar)^2}-1
\leq  \frac{\|\vec w\|_2^2(1-r^2\sin^2\theta)}{(\vec w\cdot \wstar)^2}-1
         =\frac{1-r^2\sin^2\theta}{\cos^2\theta}-1= (1-r^2) \tan^2\theta\;.
    \end{align*}
We note that from the above derivation, we have that whenever we use the update rule on mistakes, it holds that $\tan(\theta')\leq \tan(\theta)$. We show that with constant probability, the decrease is significantly larger.
For this sketch, we assume that $n \theta$ is not exponentially large 
and refer to \Cref{app:uniform} for details. 
From \Cref{pro:informal-unconditional-max-margin}, we have that with probability at least $2/3$, 
it holds $|\vec w\cdot \widehat{\x}|\geq \Omega(\sqrt{\log ( n \theta)/d}) \sin\theta$. 
Simplifying the expression for $\tan \theta'$ similarly to \Cref{eq:super-linear-angle-decay}, we obtain the result.
\end{proof}
We now show that given a non-increasing stochastic process $\xi_t$ that has good probability to decrease at a super-linear rate, then after $T = O(\log \log(1/\alpha))$ iterations it holds that $\xi_T \leq \alpha$.

\begin{lemma}[Super-Linear Convergence]
\label{lem:super-linear-convergence-main}
Fix $\kappa, \rho \in (0,1)$.
Consider a stochastic process $\xi_{t}$ adapted to 
a filtration $\mathcal F_t$ that satisfies:
(i) $0\leq \xi_0 \leq M$ (Bounded Initialization);
~~ (ii)  for all $t$: $0\leq \xi_{t+1} \leq \xi_{t}$ (Monotonicity);
~~  (iii) for all $t$:
\(
\pr[\xi_{t+1} \leq \xi_{t}^{(1-\rho)} \kappa^{\rho} \mid \mathcal F_t] \geq 2/3
\,
\) (Super-Linear Decay).
Then, for any $T$ larger than $ (3/2) ~
( (1/\rho) \max(\log \log(1/\kappa), \log \log (M+1))  + \log(e/\delta) ) $,  
with probability at least $1-\delta$, it holds that 
$\xi_{T} \leq e^2 \kappa $.
\end{lemma}

\subsubsection{Putting Everything Together: Proof of \Cref{thm:sd-learning-sphere}}
For this sketch we shall assume that we have an initialization
$\vec w$ such that $\theta(\vec w, \vec w^\ast)$ is sufficiently small. Let $T= c d\log\log n\log(1/\delta)$ for some sufficiently large absolute constant $c>0$. We split $X$ into $2k$ subsets $U_1,\ldots, U_{2k}$, with $k=n/(2T)$ and let $N=n/(2T)$ be the number of samples in each bucket. We assume that $N$ is greater than a sufficiently large constant; otherwise, $N\leq O(T)$ and the mistake bound would be at most $O(T)$. Note that each set $U_i$ is independent of all others.

Let $\vec w^{(0)}=\vec w$ and $\vec u^{(0)}=\vec w$. We analyze algorithm \Cref{alg:sd-margin-perceptron-sphere-main} for $\vec w^{(0)}$ (as the analysis of
$\vec v^{(0}$ is similar). Step 5 of \Cref{alg:sd-margin-perceptron-sphere-main} runs Margin-Perceptron in each set and goes to the next set when a mistake occurs. Let $\vec w^{(t)}$ be the current hypothesis and $\theta^{(t)}=\theta(\vec w^{(t)},\wstar)$.
From \Cref{lem:update-no-buckets-main}, conditioned on $\vec w^{(t)}$, we have that if a mistake occurred, then we construct a new vector $\vec w^{(t+1)}$ with $\theta^{(t+1)}=\theta(\vec w^{(t+1)},\wstar)$ so that $\tan\theta^{(t+1)}\leq \tan\theta^{(t)}$ and furthermore with probability at least $2/3$ we have that
\(
\tan(\theta^{(t+1)})\leq \tan^{1-1/(8d)}\theta^{(t)} ({C''}/{N})^{1/(8d)}\;,
\)
where $C''>0$ is an absolute constant.
Let $\xi_{t}=\tan\theta^{(t)}$. We have that $0\leq \xi_0= \tan\theta^{(0)}\leq 1$ and that $\xi_{t+1}\leq \xi_t$. Hence, we have that
\(
\pr[\xi_{t+1}\leq \xi_t^{1-(1/8d)} ({C''}/{N})^{1/(8d)}|\vec w^{(t)}]\geq 2/3\;.
\)
Therefore, using \Cref{lem:super-linear-convergence-main}, we get that $\theta_{T}\leq O(1/N)$, with probability at least $1-\delta/4$. 
Therefore, in Step 5a of \Cref{alg:sd-margin-perceptron-sphere-main}, the algorithm made at most $M_1=2T$ mistakes. Next, we bound the number of mistakes in Step 6a. Note that $A=\cup_{i=k}^{2k}U_i$, contains $\Omega(n)$ samples. From \Cref{lem:generalization}, we have that with probability at least $1-\delta/4$ conditioned on the event that $\theta^{(T)}\leq O(T/n)$, \Cref{alg:sd-margin-perceptron-sphere-main}, labels the points in $A$, with at most $O(T)$ mistakes. The same arguments show the same for the hypothesis $\vec u^{(T)}$. Therefore, the number of mistakes is at most $O(d\log\log( n))\log(1/\delta)$, with probability at least $1-\delta$. 
     
\section{Self-Directed Learning on Arbitrary Datasets}
\label{sec:forster}
In this section, we prove our result for self-directed classification
for arbitrary datasets. 
We first state the formal version of \Cref{inf-thm:sd-distribution-free}.

\begin{theorem}[Strong, Self-Directed Learner for Arbitrary Data]
\label{thm:strong-distribution-free-sd}
    Let $\mathcal{C}$ be the class of LTFs on $\R^d$ and let $X$ be a set of $n$ unlabeled points in $\R^d$. There exists a algorithm that runs in
    $\poly(d, n)$ time, makes 
    $\wt{O}(d^2 \log( d / (\epsilon \delta) ) )$ mistakes, and, with probability at least $1-\delta$, correctly classifies a $(1-\epsilon)$-fraction of the points of $X$.  
\end{theorem}

\subsection{Roadmap of the Proof \Cref{thm:strong-distribution-free-sd}}
\label{sec:roadmap-distribution-free-sd}

\paragraph{Boosting a Weak Self-Directed Learner}
The main ingredient in the proof of \Cref{thm:strong-distribution-free-sd} is 
a \emph{weak-learner} that does $O(d \log d)$ mistakes and correctly labels roughly $\Omega(1/d)$-fraction of the dataset $X$ with non-trivial (say above $1\%$)
probability of success.  We show the following proposition.

\begin{proposition}[A Weak, Self-Directed Learner for Arbitrary Data]
\label{pro:weak-distribution-free-sd}
    Let $\mathcal{C}$ be the class of LTFs on $\R^d$ and let $X$ be a set of $n$ unlabeled points in $\R^d$. There exists a universal constant 
    $c$ and an algorithm that runs in $\poly(d, n)$ time, makes 
    $O(d \log d)$ mistakes, and, with probability at least $c$, correctly classifies an $\Omega(1/d)$-fraction of the points of $X$.  
\end{proposition}
We give a generic boosting algorithm that allows one to obtain a strong learner
and prove \Cref{thm:strong-distribution-free-sd}. At a high-level
one can iteratively use the weak-learner to label a fraction of points, 
remove it from the dataset, and reuse the weak-learner on the remaining data. 
\begin{lemma}[Boosting]
\label{lem:boosting}
Let $\mathcal{A}$ be a distribution-free self-directed learner 
that makes $M$ mistakes and correctly labels a $(1-\alpha)$-fraction of $X$ 
for some fixed $\alpha \in (0,1)$, with probability at least $c\in(0,1)$.  Then, there exists a strong 
self-directed learner that makes $\wt{O}((M/c) ~ \log(1/(\delta \epsilon)) / \log(1/\alpha) )$ 
mistakes and labels $(1-\epsilon)$-fraction of $X$
with probability at least $1-\delta$.
\end{lemma}
We remark, that this ``label-then-remove'' approach crucially relies 
on the weak-learner being able to handle arbitrary datasets 
(as the distribution of the remaining data 
is no-longer the same as the one that generated the data initially).
We present the details of our boosting approach in \Cref{ssec:boosting}.

\paragraph{Weak Learning via Forster Transform and Margin Perceptron}

Similarly to our algorithm for spherical data, at a high-level, 
our algorithm relies on picking the ``easiest'' examples
first, i.e., picking the samples with the maximum possible margin from the current hypothesis.  We then use the margin-perceptron update 
as we did in the distribution specific setting, see \Cref{eq:margin-perceptron}.
However, as we observed in \Cref{sec:spherical-roadmap}, picking examples
that have good margin with the current hypothesis is crucial and 
since an arbitrary dataset $X$ is not guaranteed to have margin, 
the margin-perceptron update \emph{may make small or even zero progress}.
To overcome this issue we perform a pre-processing step to ensure that the resulting dataset has \emph{soft-margin} with respect to every
halfspace while at the same time remaining linearly separable.

We observe that given any dataset $X$ one can perform an (invertible) 
linear transformation $\vec A$ on the points of $X$ and obtain a dataset that 
is still linearly separable: assuming that the initial dataset 
is separable by $\vec w^\ast$ then 
for every $\x \in X$ we have 
$\vec w^{\ast} \cdot \x =  
(\vec A^{-1} \vec w^{\ast}) \cdot  (\vec A \vec x)  $
and therefore the vector $\vec A^{-1} \vec w^\ast$ corresponds
to the normal vector of a linear separator of the transformed
dataset.  Moreover, we can preserve linear separability by rescaling
each $\x$ to lie on the unit-sphere $\x \mapsto \x/\|\x\|_2$.
Forster transform combines the two transformations for some 
invertible matrix $\vec A$, i.e., $\x \mapsto \vec A \x / \|\vec A \x\|_2$
and transforms the dataset so that it is in (approximate) Radially Isotropic
Position.  There are several efficient algorithms (see, e.g., \cite{AKS20,DKT21}) to compute such an invertible matrix $\vec A$ and more recently in \cite{DTK22}
a strongly polynomial-time algorithm for computing Forster transforms was given,
see \Cref{pro:main-algorithmic-forster-transform}.
\begin{definition}[Radially Isotropic Position]
\label{def:rip}
Let $X$ be a multiset of $n$ non-zero points of $\R^d$.
We say that $X$ is in $\delta$-approximate Radially Isotropic Position if:
\vspace{-0.3em}
\begin{enumerate}
\itemsep0em
\item For every $\x \in X$, it holds $\|\x\|_2 = 1$. (Unit Norm)
\item For any unit vector $\vec u \in \R^d$, it holds 
$(1/|X|) \sum_{\x \in X} (\vec u \cdot \x)^2 \geq 1/d - \delta$.  (Isotropic Position)
\end{enumerate}
\vspace{-0.3em}
\end{definition}
Assuming that the dataset $X$ is in Radially Isotropic Position, one can show that $X$ has ``soft-margin'' with respect to every halfspace, in the sense
that for every unit vector $\vec w$ it holds that at least $\Omega(1/d)$-fraction
of $X$ has margin $|\vec w \cdot \x| \geq \Omega(1/\sqrt{d})$, see \Cref{lem:soft-margin-distribution-free-main}.  Now that we have this ``soft-margin'' we are able
to show that the margin-perceptron will correctly label a non-trivial ($\Omega(1/d)$-fraction) part of the dataset.  We refer to \Cref{ssec:proof-weak-distribution-free-sd} and \Cref{alg:margin-perceptron-distribution-free-main} 
for more details.

\begin{Ualgorithm}
	\centering
	\fbox{\parbox{5.8in}{
 {\bf Input:} An unlabeled dataset $X \subseteq \R^d$.\\
 {\bf Output:} A sequence of labeled data $(\x^{(t)}, z^{(t)})$.
\begin{enumerate}
     \setlength\itemsep{0.1em}
    \item Find subspace $V$ of dimension $k$ so that 
    $|X \cap V| \geq (k/d) ~ n$  and
    $X \cap V$ is in $1/(2d)$-approximate Radially Isotropic Position
    using \Cref{pro:algorithmic-forster-transform}. Set $U = X \cap V$.
     \item Randomly initialize guess $\vec w^{(0)} \sim \mathbb S_k$.
    \item For $t =0,\ldots, 5k \log k$:
    \begin{enumerate}
      \setlength\itemsep{0.1em}
        \item Obtain $U_{\vec w^{(t)}}$ by sorting the points of $U$ in decreasing order of margin from 
        $\vec w^{(t)}$, i.e., $|\vec x^{(i+1)} \cdot \vec w^{(t)}| \leq
        |\vec x^{(i)} \cdot \vec w^{(t)}|$.
        \item Initialize the set of correctly predicted points $C \gets \emptyset$.
        \item For $\x \in U_{\vec w^{(t)}}$:
        \begin{enumerate}
          \setlength\itemsep{0.1em}
            \item Predict the label of $\x$ with $\vec w^{(t)}$.
            \item If the prediction is incorrect, update 
            $\vec w^{(t+1)} \gets \vec w^{(t)} - (\vec w^{(t)} \cdot \x) ~ \x$, add 
            $(\x, -\sgn(\vec w^{(t)}\cdot \x) ) $ to $C$, and exit the inner loop.
            \item If the prediction is correct, add 
            $(\x, \sgn(\vec w^{(t)}\cdot \x) ) $ to $C$. 
        \end{enumerate}
        \item If $|C| \geq |U|/(4k)$ then return $C$ and exit the loop.
    \end{enumerate}
    
\end{enumerate}
    
    }}

    \vspace{0.1cm}
	\caption{A Weak Self-Directed Learner for an Arbitrary Dataset $X$.}
	\label{alg:margin-perceptron-distribution-free-main}
\end{Ualgorithm}

\subsection{Proof of \Cref{pro:weak-distribution-free-sd}}
\label{ssec:proof-weak-distribution-free-sd}
We shall use the strongly polynomial time algorithmic result to compute
a Forster transform (or show that one does not exist) given in the recent
work of \cite{DTK22}.
\begin{proposition}[Algorithmic Forster Transform, \cite{DTK22}]
\label{pro:main-algorithmic-forster-transform}
Given a set of non-zero points $X$, and an invertible matrix $\vec A \in \R^{d \times d}$, we denote by $S_{\vec A}(X) =\{ \vec A \vec x/\|\vec A \vec x\|_2 : \x \in X\}$.
There exists an algorithm, that given a set of points $X$ in 
$\mathbb Z^d \setminus \{\vec 0\} $ and some
    $\delta > 0$, runs in time $\poly(n, d, \log(1/\delta))$ and returns a subspace $V$ of $\R^d$
    containing at least a $\dim(V)/d$-fraction of the
    points $X$ and an invertible matrix $\vec A \in \R^{d \times d}$ such that $S_{\vec A}(X\cap V)$ is 
    in $\delta$-approximate radially isotropic position.
\end{proposition}

In the next lemma we show that a dataset in (approximate) Radially Isotropic
Position, satisfies a notion of ``soft-margin'' in the sense that non-trivial
part of the dataset has non-trivial margin with respect to every halfspace. Its proof can be found on \Cref{app:forster}.

\begin{lemma}[Soft-Margin via Radially Isotropic Position]
\label{lem:soft-margin-distribution-free-main}
Let $X$ be a multi-set of non-zero points 
in $1/(2d)$-approximate Radially Isotropic Position.  
For every unit vector $\vec u \in \R^d$, we have 
$\pr_{\x \sim X}[|\vec u \cdot \x| \geq {1}/{(2 \sqrt{d})}] \geq  {1}/{(4 d)}\,.
$
\end{lemma}

 Denote by $N= |U|$ the number of points that are returned in Step 1 of \Cref{alg:margin-perceptron-distribution-free-main}, and note that $N\geq nk/d$. From Lemma 3.2.4 \cite{Ver18}, we get that with probability larger than an absolute constant, the random initialitation gives a point $\vec w^{(0)}$, so that $\vec w^{(0)}\cdot \vec v\geq 1/(2\sqrt k)$. In what follows, we condition
 on the initialization satisfying this correlation bound.
We show that if \Cref{alg:margin-perceptron-distribution-free-main} terminates, then $1/(4d)$-fraction of points is correctly classified.  Note that \Cref{alg:margin-perceptron-distribution-free-main} terminates if the algorithm makes $5d \log d$ mistakes or when $|C| \geq |U|/(4k) \geq (k/d) n /(4k)
\geq n/(4d)$ (and therefore, the algorithm classifies at least 
$1/(4d)$-fraction of $X$ correctly.
Thus the bad event is that algorithm does $5d \log d$ mistakes and $|C| < |U| / (4k)$.  We argue that this cannot happen.  Let $n_i$ be the remaining points in the $i$-th iteration. Note that $N=n_i+|C|$. We make use of the following lemma (a variant of which was shown in \cite{DunaganV04}); its proof can  be found on \Cref{app:forster}.  It shows that
when we are using the margin-perceptron update,
not many mistakes with large margin can occur.
\begin{lemma}[Margin Perceptron \citep{DunaganV04} ] \label{lem:margin-perceptron-main}
    Let $\vec w^\ast, \vec w^{(0)}\in \R^d$ be unit vectors such that $\vec w^\ast \cdot \vec w^{(0)}\geq \alpha$, for some $\alpha>0$. Assume the following: $\vec w^{(t+1)}\gets \vec w^{(t)}-\x^{(t)}(\x^{(t)}\cdot \vec w^{(t)})$ and let $t_0\in \mathbb Z_+$, so that for all $t\in \Z_+$ with $t\leq t_0$, $|\x^{(t)}\cdot \vec w^{(t)}|\geq  \beta\|\vec w^{(t)}\|_2$ and $(\x^{(t)}\cdot \vec w^{(t)}) (\x^{(t)}\cdot \vec v)<0$. Then, $t_0\leq (2/\beta^2)\log(1/\alpha)$.
\end{lemma}

Assume that after $t_1=(5 d \log d -1)$ mistakes, $|C|<|U|/(4k)$. That means for all $t \leq t_1$ it holds 
$n_t=|U|-|C|\geq n(k/d-1/(4d))\geq |U|/2$, as $d\geq1$. Let $\mathcal S_t=\{\x^{(i)}: |\vec w^{(t)}\cdot\x^{(i)} |\geq 1/(2\sqrt{k})\}$. From \Cref{lem:soft-margin-distribution-free-main}, it holds that for each $t$, $|\mathcal S_t|\geq |U|/(4k)$ and combining with the fact that $n_t\geq |U|/2$, that means that either in each iteration, the algorithm makes no mistakes in the set $\mathcal S_t$, which means that $|C|\geq |U|/(4k)$ and the algorithm terminates, or that it makes one mistake in the set $\mathcal S_t$, which means that if $\x^{(t)}$ is the vector that $\vec w^{(t)}$ made a mistake then $|\vec w\tth \cdot \x\tth|\geq 1/(2\sqrt{k})$. Hence, conditional on the event that the algorithm did not terminate before the iteration $t_0$, by \Cref{lem:margin-perceptron-main} if $t_0\geq 5d\log d$, then $\vec w^{(t_0)}$ makes no mistakes in the set $\mathcal S_{t_0}$, so it classifies correctly $|U|/(4k)$ points, and the algorithm terminates.

  \bibliographystyle{alpha}
\bibliography{mydb}

\appendix
\newpage

\section*{Appendix}

\section{Preliminaries and Notation}\label{sec:prelims}
For $n \in \Z_+$, let $[n] \eqdef \{1, \ldots, n\}$.  We use small boldface characters for vectors
and capital bold characters for matrices.  For $\bx \in \R^d$ and $i \in [d]$, $\bx_i$ denotes the
$i$-th coordinate of $\bx$, and $\|\bx\|_2 \eqdef (\littlesum_{i=1}^d \bx_i^2)^{1/2}$ denotes the
$\ell_2$-norm of $\bx$.  We will use $\bx \cdot \by $ for the inner product of $\bx, \by \in \R^d$
and $ \theta(\bx, \by)$ for the angle between $\bx, \by$.  We slightly abuse notation and denote
$\vec e_i$ the $i$-th standard basis vector in $\R^d$.  We will use $\1_A$ to denote the
characteristic function of the set $A$, i.e., $\1_A(\x)= 1$ if $\x\in A$ and $\1_A(\x)= 0$ if
$\x\notin A$.
We use the standard $O(\cdot), \Theta(\cdot), \Omega(\cdot)$ asymptotic notation. We also use
$\wt{O}(\cdot)$ to omit poly-logarithmic factors. 
We use $\E_{x\sim \D}[x]$ for the expectation of the random variable $x$ according to the
distribution $\D$ and $\pr[\mathcal{E}]$ for the probability of event $\mathcal{E}$. To simplicity notation, we may omit the distribution when it is clear from the context.  For a set $X$ we use the $\x \sim X$ to denote sampling $\x$ 
uniformly at random from $X$.  For example, $\x \sim \mathbb S_d$ means that we sample
$\x$ uniformly at random from the $d$-dimensional unit sphere.

\section{Random-Order Learners Make $\Omega(d \log n)$ Mistakes}
In this section we show that random- and worst-order learners make
at least $\Omega(d \log n)$ mistakes.  This is true even for weak learning
(i.e., labeling only $1\%$ of the dataset) and even when the dataset $X$
is drawn i.i.d.\ from the unit sphere $\mathbb S_d$.  The proof
relies on a distribution specific (for $\mathbb S_d$) 
PAC learning lower-bound given in \cite{Long:95}.
\begin{proposition}[Mistake Lower Bound for Random-Order]
\label{pro:random-order-lower-bound}
Let $\x^{(1)}, \ldots, \x^{(n)}$ 
be a set of $n$ i.i.d.\ samples 
from $\mathbb S_d$ with ground-truth labels given 
by some halfspace with normal vector $\vec w^\ast$, i.e.,
the label of $\x^{(i)}$ is $\sgn(\vec w^\ast \cdot \x^{(i)})$.
Then any algorithm that predicts
the labels of $\x^{(1)},\ldots, \x^{(n)}$ in random order
makes at least $\Omega(d \log n)$ mistakes
in expectation.
Moreover, this is true even if the labeling algorithm predicts
labels for only $1\%$ of the samples $\x^{(1)},\ldots, \x^{(n)}$.
\end{proposition}
\begin{proof}
We consider the time $t$ in the labeling algorithm, i.e.,
the algorithm has predicted (and therefore also observed 
the correct labels) of a random subset of $t$ examples.
Since all points $\x^{(1)}, \ldots, \x^{(n)}$ are drawn i.i.d.\
from the uniform distribution on the unit sphere we have 
that any random subset of $t$ points is also an i.i.d.\ sample
of uniformly random points on the sphere.  We are going to
show that any algorithm that has observed the labels of 
the random subset of size $t$, makes a mistake on the next
example (that is also a uniformly random sample on the
unit sphere) with probability at least $\Omega(d/t)$.
Although this is generally given by standard VC bounds since
our distribution is uniform on the sphere, we require the following
result from \cite{Long:95}.
In what follows we shall denote by $\mathcal{A}(\x;S)$ the prediction of 
some generic learning algorithm $\mathcal{A}$ on an example $\x$ given
a labeled dataset $S$.  When the training dataset is clear from the context
we may also simply write $\mathcal A(\x)$.
\begin{lemma}[PAC Learning Halfspaces on the Unit Sphere \citep{Long:95}]
\label{lem:pac-learning-uniform-sphere}
Fix a ground-truth halfspace $f(\x) = \sgn(\vec w^\ast \cdot \x)$
for some weight vector $\vec w^\ast \in \R^d$.
Let $\x^{(1)}, \ldots, \x^{(t)}$ be a set of $t$ i.i.d.\ samples
drawn uniformly at random on the unit sphere.  The expected error of
any learning algorithm $\mathcal A$ that has observed 
$\x^{(1)},\ldots, \x^{(t)}$ (and their ground-truth labels)
is at least $\Omega(d/t)$ 
\[
\E_{\x^{(1)}, \ldots, \x^{(t)} \sim \mathbb S_d}
\left[
\pr_{\x \sim \mathbb S_d}[ \mathcal{A}
\Big(\x; (\x^{(1)}, f(\x^{(1)})),
\ldots, (\x^{(t)}, f(\x^{(t)}) ) \Big)
\neq f(\x)] \right]  \geq  c \frac{d}{t}
\,,
\]
where $c$ is some universal constant.
\end{lemma}
Using \Cref{lem:pac-learning-uniform-sphere}, we obtain that
after predicting the labels on $t$ examples, the expected
probability that any algorithm makes an incorrect prediction
on a fresh example is at least $2/3$.  
Given any prediction algorithm $\mathcal{A}$, we define the 
error of $\mathcal{A}$ to be the probability that $\mathcal{A}$
makes an incorrect prediction on a fresh sample from $\mathbb S_d$,
i.e., $\err(\mathcal{A}, \x) = \1\{\mathcal A(\x) \neq \sgn(\vec w^\ast \cdot \x)\}$.
To simplify notation, we shall denote by $S_t = \{(\x^{(1)}, f(\x^{(1)})),
\ldots, (\x^{(t)}, f(\x^{(t)}))\}$ a set of $t$ labeled examples
and by $\mathcal F_t$ the corresponding filtration (so that $S_t$ 
is adapted to $\mathcal F_t$).  
We have that
\[
\E\Big[\sum_{t=1}^n \err(\mathcal{A}(\cdot; S_t) , \x^{(t+1)})\Big] 
=
\sum_{t=1}^n \E[\err(\mathcal{A}(\cdot; S_t) , \x^{(t+1)})\Big] 
\geq \sum_{t=1}^n c \frac{d}{t} 
\geq c d \log n \,,
\]
where for the last inequality, we used
the fact that the harmonic number $\sum_{t=1}^n 1/t = \Omega(\log n) $.
Finally, we see that the same is true if we only label only $1\% n$ points
since  $\sum_{t=1}^{0.01 n} 1/t = \Omega(\log n)$.

\end{proof}

 \section{Self-Directed Learning on $\mathbb S_d$: Proof Details    }\label{app:uniform}  
 \begin{Ualgorithm}
	\centering
	\fbox{\parbox{5.7in}{
 {\bf Input:} An initialization $\w$. 
 {\bf Output:} A sequence of labeled data $(\x^{(t)}, z^{(t)})$.
\begin{enumerate}
  \setlength\itemsep{0.1em}
\item Initialize guesses $\vec w^{(0)} \gets \vec w$,
~~~
$\vec v^{(0)} \gets \vec w$.
\item Initialize the set of unlabeled data $U \gets X$, $t \gets 0$.
\item Split $U$ in $2 k$ sets $U_{1},\ldots, U_{2k}$.

\item For $t =1,\ldots, k$:
\smallskip
\\
$~~\vec w^{(t)} \gets$ \textsc{Margin-Perceptron}($U_t$, $\vec w^{(t-1)}$), 
$\vec v^{(t)} \gets$ \textsc{Margin-Perceptron}($U_{k + t}$, $\vec v^{(t-1)}$).
\item For $t =1,\ldots, k$:
\\
$~~$ Label points of $U_{k + t}$ with $\vec w^{(k)}$
    and  label points of $U_t$ with $\vec v^{(k)}$.
\end{enumerate}
\centering
	\fbox{\parbox{5.4in}{
\textsc{Margin-Perceptron}($U,\vec w$)\\
 {\bf Input:} An initialization $\w$ and a set of points $U$. 
 {\bf Output:} A vector $\vec w'$.
\begin{enumerate}
  \setlength\itemsep{0.1em}
\item Obtain $U'$ by sorting the points of $U$ in decreasing order of margin from 
        $\vec w$, i.e., $|\vec x^{(i+1)} \cdot \vec w| \leq
        |\vec x^{(i)} \cdot \vec w|$.
\item For $\x\in U'$:
\begin{enumerate}
  \setlength\itemsep{0.1em}
\item Predict the label of $\x$ with $\vec w$.
\item If the prediction is incorrect, 
exit the loop and return $\vec w' \gets \vec w-(\vec w\cdot \x)\x$.
\end{enumerate}
\end{enumerate}}}
    }}
    \vspace{0.1cm}
	\caption{Self-Directed Learning on $\mathbb S_d$}
	\label{alg:sd-margin-perceptron-sphere}
\end{Ualgorithm}

\subsection{Proof of \Cref{thm:sd-learning-sphere}}
We restate and prove \Cref{thm:sd-learning-sphere} in this section.
\begin{theorem}
Let $\delta \in (0,1/2]$ and let $n$ be larger than 
some sufficiently large universal constant.
Let $X$ be a set of $n$ i.i.d.\ samples from $\mathbb S_d$
with true labels given by a homogeneous halfspace,
$f(\x) = \sgn(\vec w^\ast \cdot \x)$.
There exists a self-directed classifier that makes $O(d \log \log n ~  \log(1/\delta))$ mistakes, runs in time $\poly(d,n)$ and classifies 
all points of $X$ with probability at least $1-\delta$.
\end{theorem}

We first show our anti-concentration result for the maximum-margin 
of the conditional distribution on the disagreement region $C$.
We believe that our tight anti-concentration bound is of independent
interest and may find other applications in convex geometry and 
learning linear classifiers.

\begin{proposition}
\label{pro:conditional-max-margin}
Let $C$ be the indicator of the disagreement region of two homogeneous
halfspaces, i.e., $C = \1\{ (\vec v \cdot \x) 
(\vec u \cdot \x) \leq 0\}$ for some unit vectors
$\vec v, \vec u \in \R^d$ with angle $\theta(\vec v, \vec u) = \theta$.
Let $\mathbb S_d$ denote the uniform distribution on the unit 
sphere and by $\mathbb S_d(C)$ the conditional
distribution on the disagreement region $C$.
\begin{enumerate}
\item 
For any $\alpha \in [0, 1]$, it holds:
\[
\pr_{\x^{(1)},\ldots, \x^{(m)} \sim \mathbb S_d(C)} 
\left[\max_{i=1,\ldots,m} |\vec u \cdot \x^{(i)}| \leq 
\alpha\sin(\theta/2)
\right] \leq \exp\left(-m ~  (1-\alpha^2)^{d/2-1}/2\right)  \,.
\]
\item 
For any $\beta \in [0, 1]$, it holds:
\[
\pr_{\x^{(1)},\ldots, \x^{(m)} \sim \mathbb S_d(C)} 
\left[\max_{i=1,\ldots,m} |\vec u \cdot \x^{(i)}| \leq 
(1-\beta) \sin(\theta)
\right] \leq 
\exp\Big(-m ~ (\beta/2)^{d/2}/2 \Big)
\,.
\]
\end{enumerate}
\end{proposition}

\begin{proof}
To simplify notation, set $\Delta = \alpha \sin (\theta)$.
We first compute the probability that a single sample in $C$ 
has $|\vec u \cdot \x| \geq \Delta$.
By the symmetry of the set $C$ and the uniform distribution on the sphere $\mathbb{S}_d$ 
\[
\pr_{\x \sim \mathbb S_d}[ \vec x \in C, |\vec u \cdot \x| \geq \Delta]
=
2 
\pr_{\x \sim \mathbb S_d}[ \vec x \in C, \vec u \cdot \x \geq \Delta] 
=
2 
\pr_{\x \sim \mathbb S_d}[E_1]\, .
\]
where $E_1 = \{\x : \vec x \in C, \vec u \cdot \x \geq \Delta\}$.
Assume, without loss of generality that $\vec u = \vec e_2$
and $\vec v = - \sin \theta \vec e_1 + \cos \theta \vec e_2$.
Observe that the set $E_1$ can now be written as
$E_1 =\{ (\x_1, \x_2): \x_2 \geq \Delta,  \cos \theta \x_2 \leq \sin \theta \x_1 \}$.  
Using polar
coordinates $\x_1 = r \cos \phi$,
$\x_2 = r \sin \phi$ we have that 
$E_1 = \{ (r, \phi): 0 \leq r \leq 1, r \sin \phi \geq \Delta,
r \cos\theta \sin \phi \leq r \sin \theta  \cos\phi\}=
\{ (r, \phi) : 0 \leq r \leq 1, r \sin \phi \geq \Delta,
\sin(\theta-\phi) \geq 0
\}$, where we used the trigonometric identity $
    \sin(\phi - \theta) = \sin \phi \cos \theta - \cos \phi \sin \theta
$ and the fact that $\sin(-z) = - \sin(z)$.  
Moreover, by $r \sin \phi \geq \Delta$ we obtain
that $\sin \phi \geq 0$ and therefore $\phi \in [0, \pi]$.
Combining this with the fact that the angle between two
halfspaces can be at most $\theta \leq \pi$, we obtain
that $\sin(\theta - \phi) \geq 0$ implies that 
$\phi \leq \theta$.
Finally, observe that the constraint 
$r \sin \phi \geq \Delta =
\alpha \sin \theta$ implies that the radius $r \geq \alpha$.
Therefore the set $E_1$ can be equivalently written as
\[
E_1 = \{(r,\phi): \alpha \leq r \leq 1, r \sin \phi 
\geq \alpha \sin \theta, \phi \leq \theta\}
\]
The set $E_1$ has coupled
constraints (i.e., constraints that depend on both $r, \phi$).
The set 
$E_2 = \{(r, \phi): \alpha \leq r \leq 1, 
\theta/2 \leq \phi \leq \theta  \} $ has decoupled constraints and is a subset of $E_1$.  To see that $E_2 \subseteq E_1$, notice that
for $(r, \phi) \in E_2$ it holds 
$r \sin \phi \geq \alpha \sin(\theta/2) = \Delta$
and $\theta - \phi \in [0, \pi/2]$ which implies that $
\sin(\theta - \phi) \geq 0$ and therefore, $(r, \phi) \in E_1$.
We can now directly estimate the probability of the set $E_1$.
The $2$-dimensional projection of the uniform on the sphere 
has density $\frac{d-2}{2 \pi} (1-r^2)^{d/2-2} r$ (in polar coordinates).  
\begin{align*}
\pr_{\x \sim \mathbb S_d}[E_2]
&=\frac{d-2}{2 \pi} \int_{\alpha}^1 \int_{\theta/2}^\theta 
(1-r^2)^{d/2-2} ~ r ~ d \phi d r
=
 \frac{\theta (d-2)}{4 \pi} 
 \int_{\alpha}^1 (1-r^2)^{d/2-2} ~  r
 ~ d r
 \\
 &=
 \frac{\theta}{4 \pi} (1-\alpha^2)^{d/2-1}
 \,.
\end{align*}

By the symmetry of $\mathbb S_d$ we directly obtain that
$\pr_{\x \sim \mathbb S_d}[C] = \theta /\pi$.
We conclude that the conditional probability 
$\pr_{\x \sim \mathbb S_d(C)}[ |\vec u\cdot \x| \geq \Delta]
\geq (1/2) (1-\alpha^2/d)^{d/2-1}$.
We can now bound above the probability that the maximum of 
$m$ independent samples from $\mathbb S_d(C)$ is small.
\begin{align*}
\pr_{\x_1,\ldots, \x_m \sim \mathbb S_d(C)} 
&\left[\max_{i=1,\ldots,m} |\vec u \cdot \x_i| \leq \Delta
\right] 
= (1- \pr_{\x \sim \mathbb S_d(C)}[|\vec u \cdot \x| \geq \Delta])^m
\\
&\leq \exp\left(-m \pr_{\x \sim \mathbb S_d(C)}[|\vec u \cdot \x| \geq \Delta]\right)
\leq  \exp\left(-m ~ (1-\alpha^2)^{d/2-1}/2\right) \,,
\end{align*}    
where, for the first inequality, we used the fact $e^{x} \geq 1 + x$.

We now prove the second inequality that allows us to achieve
correlation arbitrarily close to $\sin \theta$ albeit 
with worse success probability.
To keep the proof similar to the previous one, we shall use continue 
using the parameter $\alpha = 1-\beta$ and replace it with $\beta$
in the final expression for the probability.
Recall that the expression of the set $E_1$ in polar coordinates is
\[
E_1 = \{(r,\phi): \alpha \leq r \leq 1, r \sin \phi 
\geq \alpha \sin \theta, \phi \leq \theta\}\;.
\]
This time, we estimate directly the probability of $E_1$.
To simplify notation, set $q = \alpha \sin \theta$. 
We have:
\begin{align*}
\pr_{\x \sim \mathbb S_d}[E_1]
&=\frac{d-2}{2 \pi} \int_{\sin^{-1}(q)}^\theta 
\int_{\frac{q}{\sin \phi}}^1 
(1-r^2)^{d/2-2} ~ r ~ d r d \phi
=
\frac{1}{2 \pi}
 \int_{\sin^{-1}(q)}^\theta 
\Big(1 - \big(\frac{q}{\sin \phi} \big)^2\Big)^{d/2-1}
 ~ d \phi\;.
\end{align*}
Since the quantity inside the integral is positive, we can 
bound its value from below by slightly increasing the 
lower threshold
to $\sin^{-1}(s\sin \theta) $ 
for $s = (1+\alpha)/2$
(where 
we used that $\sin^{-1}(\cdot)$ is increasing.
We have 
\begin{align*}
\pr_{\x \sim \mathbb S_d}[E_1]
&\geq 
\frac{1}{2 \pi} \int_{\sin^{-1}( s \sin \theta) }^\theta 
\Big(1 - \big(\frac{q}{\sin \phi} \big)^2 \Big)^{d/2-1}
d \phi
\geq 
\frac{1}{2 \pi} \int_{\sin^{-1}( s \sin \theta) }^\theta 
\Big(1 - \big(\frac{\alpha}{s} \big)^2 \Big)^{d/2-1}
d \phi
\,.
\end{align*}
Finally, observe that since $s = (1+\alpha)/2$, it holds that 
$a/s \leq s$ and therefore:
\begin{align*}
\pr_{\x \sim \mathbb S_d}[E_1]
\geq 
\frac{1}{2 \pi} (1-s^2)^{d/2-1} ( \theta - \sin^{-1}(s \sin \theta))
\geq 
\frac{\theta}{2 \pi}\frac{1-\alpha}{2} (1-s^2)^{d/2-1} ( \theta - \sin^{-1}(s \sin \theta)) \,.
\end{align*}
where, for the last inequality, we used the inequality
$\sin^{-1}(\alpha x) \leq \alpha \sin^{-1}(x)$.  
Therefore, we have proved the bound 
\[
\pr_{\x \sim \mathbb S_d}[E_1]
\geq 
\frac{1}{2} s (1-s^2)^{d/2-1} 
=
\frac{1-\alpha}{4} \left(1-\left(\frac{1+\alpha}{2}\right)^2\right)^{d/2-1} 
\,.
\]
We can now switch back to using the parameter $\beta = 1-\alpha$
to obtain the bound
\[
\pr_{\x \sim \mathbb S_d}[E_1]
\geq
\frac{\beta}{4} (1-(1-\beta/2)^2)^{d/2-1} 
\geq (\beta/2)^{d/2} / 2\,,
\]
where we used the inequality $1 - (1-x)^2 \geq x$
for all $x \in [0,1]$.
The final steps to obtain the upper bound for the probability that the maximum is small are the same 
as those of the previous case.

\end{proof}

Using \Cref{pro:conditional-max-margin} we now show considering the original
$n$ i.i.d.\ samples from $\mathbb S_d$ the maximum-margin of those
that fall in the disagreemeent region is going to be significantly larger
than that of a random sample of $\mathbb S_d$.  This is the formal
version of \Cref{pro:informal-unconditional-max-margin}.

\begin{lemma}
\label{cor:unconditional-max-margin}
Let $\vec v, \vec u \in \R^d$ be unit vectors with angle $\theta(\vec v, \vec u) = \theta$.
Let $C$ be the indicator of the disagreement region of the two homogeneous
halfspaces defined by  $\vec v, \vec u$, i.e., $C = \1\{ (\vec v \cdot \x) 
(\vec u \cdot \x) \leq 0\}$.
Furthermore, let $\mathbb S_d$ denote the uniform distribution on the unit sphere.
\begin{enumerate}
\item 
For all $n, s \geq 1$ and $c \geq 2$ such that 
$e^{-d c/4} \leq 4 \pi s/(n \theta) \leq 1$, it holds
\[
\pr_{\x^{(1)},\ldots, \x^{(n)} \sim \mathbb S_d} 
\left[\max_{i=1,\ldots,n} |\vec u \cdot \x^{(i)}| \1\{\x^{(i)} \in C\}
\leq \sqrt{ \frac{\log (n \theta/(4 \pi s) ) }{2 c ~ d} } \sin(\theta) \right]  
\leq 2 e^{-s/2}\,. \]
    \item 
For all $n, s\geq 1$ such that 
$4 \pi s/ (n \theta) \leq 1$ it holds:
it holds 
\[
\pr_{\x^{(1)},\ldots, \x^{(n)} \sim \mathbb S_d} 
\left[\max_{i=1,\ldots,n} |\vec u \cdot \x_i| \1\{\x^{(i)} \in C\}
\leq \Big(1- \big(\frac{4 \pi s}{n \theta}\big)^{2/d}\Big) \sin(\theta) \right]   \leq 2 e^{-s/2}\,. \]
\end{enumerate}
\end{lemma}

\begin{proof}
Denote by $S$ the set of samples that fall in the disagreement
region $C$ and denote by $m = |S|$ the number of samples that fall
in the disagreement region.
We observe that by the rotational symmetry of the uniform distribution on the sphere, the probability of the set $C$ is exactly 
$\theta/\pi$.  Therefore, on expectation, out of the $n$ samples that
we draw from $\mathbb S_d$, the number of samples that fall in $C$ is 
$\E[|S|] = \mu = n \theta/\pi$. We first show that with high-probability 
we are going to observe at least $\mu/2$ samples in $C$.
Denote by $m$ the number of samples that fall in $C$.
Using Chernoff's bound, we obtain that 
$\pr[m \leq \mu/2] \leq e^{-\mu/8} \leq e^{-n \theta/8} \leq e^{-4 \pi s/8}
\leq e^{-s/2}$.
Therefore, from now on, we condition on the event that at least
$m \geq n \theta/(2\pi)$ samples fall in $C$.  
In other words, out of the $n$ original samples from $\mathbb S_d$,
with probability at least $1-e^{-s/2}$, we have drawn
at least $n \theta/ (2 \pi)$ samples from the conditional distribution
$\mathbb S_d(C)$.

We first prove the second case of \Cref{cor:unconditional-max-margin}.
Using the second case of \Cref{pro:conditional-max-margin} we have that
for $\beta = (4 \pi s / (n \theta))^{2/d}$, it holds that
$ \pr[\max_{i = 1,\ldots, m} |\vec u \cdot \x^{(i)}| \leq (1-\beta) \sin \theta ]
\leq \exp(-m  ~ 2 \pi s /(n \theta))$.  Since we have conditioned
on the event that $m \geq n \theta/ (2 \pi)$, we obtain that
this probability is at most $e^{-s}$.  Combining this probability
with the rejection sampling failure probability 
(that the number of conditional samples, i.e., those that fell in $S$,
is smaller than $n \theta/(2 \pi)$), we obtain that the total
probability of failure is at most $2 e^{-s/2}$.

We now prove the first case of \Cref{cor:unconditional-max-margin}.
Observe first that for $n$, using the fact that 
$\sin(\theta)/2 \leq \sin(\theta/2)$ it holds that
\[
\sqrt{
\frac{\log(n \theta/(4 \pi s))}{ 2c~ d}} \sin \theta\leq 
\sqrt{\frac{2 \log(n \theta/(4 \pi s))}{ c~ d} }  \sin(\theta/2)
\,.
\]
At this point, notice that by the assumption of the first case of \Cref{cor:unconditional-max-margin} that $e^{-cd/4} \leq n \theta/(4 \pi s)$,
we have that $2 \log(n \theta/(4 \pi s))/ (c d) \leq 1/2$.
In particular, we never ask for margin larger than $\sin(\theta/2)/2$ 
(notice that the maximum possible margin is always $\sin(\theta)$).
Using the first case of \Cref{pro:conditional-max-margin}, we have that 
\begin{equation}
\label{eq:max-low-regime-bad-bound}
\pr\left[\max_{i=1,\ldots, m} |\vec u \cdot \x^{(i)}| \leq 
\sqrt{\frac{\log(n \theta/(4 \pi s))}{ 4 d}} \sin \theta
\right]
\leq \exp\Big(-(m/2) ~ \Big(1- 2 \frac{\log(n \theta/(4 \pi s))}{c d} \Big)^{d/2-1}
\Big) 
\,.
\end{equation}
Next, we will use the inequality
$1-x \geq e^{-2 x}$ that holds for all $x \in [0,1/2]$,
to obtain that
\[ \left(1- \frac{2 \log(n\theta/(4 \pi s))}{c d} \right)^{d/2-1} \geq 
\exp\Big(-\frac{2 d - 4}{ c d}  \log(n \theta/(4 \pi s) ) \Big)
\geq \frac{4 \pi s} {n \theta} \,,
\]
where for the first inequality we used the fact that 
$2 \log(n\theta/(4 \pi s))/(c d) \leq 1/2$ (so that we
are able to use the inequality $1-x \geq e^{-2 x}$),
and for the second inequality the fact that $c \geq 1$.
Using the fact that $m \geq n\theta/(4 \pi)$ we have that
this probability of \Cref{eq:max-low-regime-bad-bound} 
is at most $e^{-s/2}$.  Combining this failure probability
with the probability that we do not observe at least 
$n \theta/(2 \pi)$ samples in $C$ we obtain that the total
failure probability is at most $2e^{-s/2}$.
\end{proof}

We prove the following lemma, showing on each mistake
the margin-perceptron has good probability of significantly 
(super-linearly) decreasing $\tan \theta$.  We require that 
the current guess $\vec w$ is not exactly orthogonal with the target
$\vec w^\ast$ (notice the assumption $1/\cos \theta \leq O(\zeta)$).
We show that it is not hard to obtain such an initialization.

\begin{lemma}\label{lem:update-no-buckets}
  Let $\vec w \in \R^d$ and let $\theta(\vec w^\ast,\vec w)=\theta \in [0, \pi/2)$ and assume that for some parameter $\zeta>0$, 
  $1/\cos\theta \leq \zeta/(12 \pi)$. Let $C$ be the indicator of the disagreement region of $\vec w^\ast,\vec w$, i.e., $C = \1\{ (\vec w^\ast \cdot \x) 
(\vec w \cdot \x) \leq 0\}$. Let $ X=\{\x^{(1)},\ldots,\x^{(n)}\}$ be a sample set drawn from $\mathbb{S}_d$ with $n$ larger than a sufficiently large constant and let $\widehat{\x}=\argmax_{\x\in \mathcal X}  |\vec w \cdot \x| \1\{ \x \in C\}$. Denote $\vec w' =\vec w' -
(\vec w\cdot \widehat{\x})\widehat{\x}$ and let $\theta'=\theta(\vec w',\vec w^\ast)$. 
\begin{enumerate}
\item  We have that $\tan^2\theta'\leq \tan^2\theta$. (Monotonicity)
\item  With probability at least $2/3$, it holds 
\(
\tan\theta'\leq \tan^{1-1/(8d)}\theta ({\zeta}/{n})^{1/(8d)}\;.
\)
\end{enumerate}
\end{lemma}
\begin{proof}
First, we claim that this update rule decreases $\tan\theta$.
\begin{claim}
    \label{lem:update}
    Let $\vec w \in \R^d$ and let $\x\in \{\x: (\vec w^\ast \cdot \x) 
(\vec w \cdot \x) \leq 0\}$. Furthermore, assume that $|\vec w \cdot \x|\geq r\sin\theta$, where $\theta(\vec w^\ast,\vec w)=\theta\in[0,\pi/2)$ and $r>0$. Denote $\vec w'
=\vec w-(\vec w\cdot \x)\x$ and let $\theta'=\theta(\vec w',\vec w^\ast)$. Then
$\tan^2\theta'\leq \tan^2\theta(1-r^2)\;.
$
\end{claim}
\begin{proof}
 We first observe that the correlation with $\vec w^\ast$ does not decrease,
\(
            \w'\cdot\wstar=(\w-(\w \cdot\x)\x)\cdot \wstar =\w\cdot\wstar -(\w
    \cdot\x)(\x\cdot \wstar )\geq \w\cdot\wstar
    \), where we used that the hypothesis $\vec w$ disagrees with the ground-truth
      $\vec w^\ast$ on $\x$, i.e., $(\w \cdot\x)(\x\cdot \vec w^\ast)\leq 0$. 
    Furthermore, since $\vec w\cdot \vec w^\ast \geq 0$ (by the assumption that 
    $\theta \in [0,\pi/2]$) we also have that 
    $(\vec w'\cdot \wstar)^2 \geq (\vec w\cdot \wstar)^2$.
We next show that the norm of $\w'$ decreases multiplicatively.
    We have that
\(
       \|\w'\|_2^2= \|(\w-(\w \cdot\x)\x)\|_2^2=\|\w\|_2^2-(\w \cdot\x)^2
        \leq \|\w\|_2^2(1-r^2\sin^2\theta)
        \,.
        \)
Using (twice) the trigonometric identity $\tan^2\phi =(1/\cos^2\phi)-1$,  
    we show that $\tan^2\theta$ decreases by a factor of $(1-r^2)$:
    \begin{align*}
        \tan^2\theta'&=\frac{\|\vec w'\|_2^2}{(\vec w'\cdot \wstar)^2}-1
\leq  \frac{\|\vec w\|_2^2(1-r^2\sin^2\theta)}{(\vec w\cdot \wstar)^2}-1
         =\frac{1-r^2\sin^2\theta}{\cos^2\theta}-1= (1-r^2) \tan^2\theta\;.
    \end{align*}
\end{proof}

We note that from \Cref{lem:update} whenever we use the update rule on mistakes, it gives that $\tan(\theta')\leq \tan(\theta)$. We show that with constant probability, we can argue that the decrease is significantly larger.

We split our analysis into two cases, one that $12\pi /(n\theta)\leq  \exp(-d/2)$ and the later case is when $1\geq12\pi /(n\theta)\geq  \exp(-d/2)$. We first consider the case where $12\pi /(n\theta)\leq  \exp(-d/2)$. From \Cref{cor:unconditional-max-margin}, we have that with probability at least $2/3$, it holds $|\vec w\cdot \widehat{\x}|\geq (1-(12\pi/(n\theta))^{2/d})\sin\theta$. Therefore, from \Cref{lem:update}, we have that
    \[
\tan\theta'\leq \tan\theta\left(1-(1-(12\pi/(n\theta))^{2/d})^2\right)^{1/2}\leq  \tan\theta\left(2(12\pi/(n\theta))^{2/d}\right)^{1/2}\;,
\]
where we used that $1-(1-x)^2\leq 2x$ for $x>0$. Note that by our assumption $12\pi /(n\theta)\leq  \exp(-d/2)$, therefore $\left(2(12\pi/(n\theta))^{2/d}\right)^{1/2}\leq (12\pi/(n\theta))^{1/(2d)}\leq (12\pi/(n\theta))^{1/(8d)}$.

Next, we consider the case where $1\geq12\pi /(n\theta)\geq  \exp(-d/2)$. In this case,  from \Cref{cor:unconditional-max-margin}, we have that with probability at least $2/3$, 
it holds $|\vec w\cdot \widehat{\x}|\geq \sqrt{\log ( n \theta/(12\pi))/(4d)} \sin\theta$. Therefore, from \Cref{lem:update}, we have that
    \begin{align*}
        \tan\theta'\leq \tan\theta\left(1-\frac{\log(\theta n/(12\pi))}{4d}\right)^{1/2}&\leq \tan\theta\exp\left(-\frac{\log(\theta n/(12\pi))}{8d}\right)
        \\&=\tan\theta\left(12\pi/(n\theta)\right)^{1/(8d)}\;.
    \end{align*}
Therefore, in both cases, with probability at least $2/3$ that $\tan\theta'\leq \tan\theta \left(12\pi/(n\theta)\right)^{1/(8d)}$. Let $\rho=1/(8d)$. We have that
\[
\tan\theta'\leq \tan^{1-\rho}\theta \left(\frac{12\pi\sin\theta}{\theta \cos\theta}\frac{1}{n}\right)^{\rho}\leq \tan^{1-\rho}(\theta) \left(\frac{\zeta}{n}\right)^{\rho}\;,
\]
where we used our assumption that $\cos \theta \geq (12\pi/\zeta)$. This completes the proof of \Cref{lem:update-no-buckets}.
\end{proof}

In the following lemma we show that given a decreasing stochastic process $\xi_t$
that has good probability to decrease in a superlinear-rate then after
$T = O(\log \log(1/\alpha))$ iterations it holds that $\xi_T \leq \alpha$.
\begin{lemma}[Super-Linear Convergence]
\label{lem:super-linear-convergence}
Fix $\kappa, \rho \in (0,1)$.
Consider a stochastic process $\xi_{t}$ adapted to 
a filtration $\mathcal F_t$ that satisfies:
\begin{enumerate}
\item 
$0\leq \xi_0 \leq M$ (Bounded Initialization),
\item 
For all $t$: $0\leq \xi_{t+1} \leq \xi_{t}$ (Monotonicity),
\item 
For all $t$:
\(
\pr[\xi_{t+1} \leq \xi_{t}^{(1-\rho)} \kappa^{\rho} \mid \mathcal F_t] \geq 2/3
\,
\) (Super-Linear Decay).
\end{enumerate}
Then, for any $T$ larger than $ (3/2) ~
( (1/\rho) \max(\log \log(1/\kappa), \log \log (M+1))  + \log(e/\delta) ) $,  
with probability at least $1-\delta$, it holds that 
$\xi_{T} \leq e^2 \kappa $.
\end{lemma}
\begin{proof}
    Define the random variable $I_t$ to be the indicator of the event that the super-linear decay step happens at step $t$, i.e.,
    that $\xi_{t+1} \leq (\xi_{t})^{(1-\rho)} \kappa^{\rho}$.
    We first observe that by the fact that the stochastic process 
    is monotone in the sense that $\xi_{t} \leq \xi_{t+1}$ for
    all $t$, it does not matter at which steps the super-linear decay
    happens (but only how many times it does so).
    Assume that $\sum_{k=1}^T I_k=m$, 
    and denote by $t_j$ be the subsequence of $\{1,\ldots, T\}$ 
    of length $m$ where the super-linear decay steps happen.
    Using the monotonicity of $\xi_{t}$, we have that
    \begin{align*}
    \xi_{T} \leq \xi_{t_m}
    &\leq (\xi_{t_m-1})^{1-\rho} \kappa^{\rho}
    \leq  (\xi_{t_{m-1}})^{1-\rho} \kappa^{\rho}
    \leq (\xi_{t_{m-1} - 1})^{(1- \rho)^2} 
    \kappa^{\rho (1-\rho) + \rho}
    \\
    &\leq (\xi_{t_{m-2}})^{(1- \rho)^2} 
    \kappa^{\rho (1-\rho) + \rho}
    \,.
    \end{align*}
    By continuing to unroll the recurrence, we obtain  
$    \xi_{T} \leq \xi_0^{(1-\rho)^m}
    \kappa^{1 - (1-\rho)^m}\,,
    $
    where we used the fact that 
    $\rho + \rho (1-\rho) + \ldots + \rho (1-\rho)^{k-1} 
    = 1 - (1-\rho)^k$.
Using the fact that $\xi_0\leq M$, we have that 
$  \xi_{T}\leq {M\kappa}^{1-(1-\rho)^m}$. We show the following in \Cref{app:uniform}.
\begin{fact}\label{fct:bound-geometric}
   If $m\geq (1/\rho) \max(\log\log(1/\kappa), \log \log(M+1))$, then 
   $M^{(1-\rho)^m}  \kappa^{1 - (1-\rho)^m} \leq e^2 \kappa$.
\end{fact}

\begin{proof}
We first show that with $m \geq \log \log(M+1)/\rho$ 
it holds that $M^{(1-\rho)^m} \leq e$.
We first observe that this is trivially true when $M\leq 1$.
For $M > 1$ we can take logarithms in both sides of 
$M^{(1-\rho)^m} \leq e$ and obtain 
$m \log(1-\rho) + \log \log M \leq 0$
or equivalently $m \geq \log \log M/\log(1/(1-\rho))$.
Since $\log(1/(1-\rho) \geq 1/\rho$ for all $\rho \in (0,1)$
we obtain that for the chosen $m$ the inequality is true.
Next we show that for $m \geq \log \log (1/\kappa)/\rho$ 
it holds that $\kappa^{1-(1-\rho)^m} \leq e \kappa$.
We first observe that we can rewrite this inequality 
as $(1/\kappa)^{(1-\rho)^m} \leq e$.
Using the same argument as in the previous case (by replacing
$M$ with $1/\kappa > 1$), we obtain the result.
\end{proof}

To complete the proof, it remains to show that 
many ``fast-decay'' updates will happen with good probability,
or, in other words, that the number $m$ defined above is at
least \[m^\ast \eqdef (1/\rho) \max(\log \log (1/\kappa), \log \log (M +1)),\]
with good probability.
We show that if the total number of updates 
$T \geq 8 \log(e/\delta) m^\ast$, then, with probability
at least $1-\delta$, $m \geq m^\ast$.
To do this, we shall Azuma's inequality for martingales.
\begin{lemma}[Azuma-Hoeffding]
\label{lem:azuma}
Let $(D_{t})$ be a martingale with bounded increments, i.e.,
$D_{t} - D_{t-1} \leq L$.
It holds that $\pr[D_T \leq D_0 - \lambda] 
\leq e^{-\lambda^2/(2 L^2 T)}$.
\end{lemma}
We define the martingale $D_T = \sum_{t=1}^T (I_t - 
\E[I_t \mid \mathcal F_{t-1}])$, with $D_0 = 0$.
Using the fact that the super-linear decay step happens
with probability at least $2/3$ (see Item 3 of \Cref{lem:super-linear-convergence}) we have that with probability at least $2/3$ 
we have  $I_t = 1$ and therefore $\sum_{t=1}^T \E[I_t \mid \mathcal F_{t-1}] \geq  (2/3) ~ T$.
Moreover, we observe that the increments of $D_T$ are bounded 
by $1$ and therefore Azuma's inequality \Cref{lem:azuma} implies
that 
$\pr[ D_T \leq - \sqrt{2 T \log(1/\delta)}] \leq  \delta \,.
$
Equivalently, we obtain that with probability at least 
$1-\delta$, it holds that the number of super-linear decay steps is
bounded below by 
    $m \geq  \sum_{t=1}^T \E[I_t \mid \mathcal F_{t-1}] \geq  (2/3) ~ T - \sqrt{2 T \log(1/\delta)} $.  For 
    $T = (3/2) m^* + (3/2) \log(e/\delta)$ we obtain
    that $m \geq m^*$ with probability at least $1-\delta$.

\end{proof}
The following lemma shows that a halfspace that has angle $\theta$ with 
the ground-truth $\vec w^\ast$ makes roughly $n\theta$ mistakes 
on a sequence of $n$ i.i.d.\ examples from the uniform distribution on the sphere. 
\begin{lemma}\label{lem:generalization}
    Fix $\vec v,\vec w\in \R^d$ and assume that $\theta(\vec v,\vec w)=\theta$. Let $C$ be the indicator of the disagreement region of $\vec v,\vec w$, i.e., $C = \1\{ (\vec v \cdot \x) 
(\vec w \cdot \x) \leq 0\}$. Let $X=\{\x^{(1)},\ldots,\x^{(n)}\}$ be a sample set drawn i.i.d.\ from $\mathbb{S}_d$. Then, with probability at least $1-\delta$, the set $X \cap C$ has size at most $O(n \theta+ \sqrt{n \theta \log(1/\delta)} ) $.
\end{lemma}
\begin{proof}
    Let $Z_i=\1\{\x^{(i)}\in C\}$. From the fact that $\pr[ (\vec v \cdot \x) 
(\vec w \cdot \x) \leq 0]=\theta/\pi$, we have that $\E[Z_i]=\theta/\pi$. We use the following version of the standard Hoeffding bound.
\begin{fact}\label{fct:hoef}
    Let $z_1,\ldots, a_n$ be i.i.d.\ random variables on $\{0,1\}$ with $\E[z_1]=p$. Then, it holds that
    \[
    \pr\left[\sum_{i=1}^n z_i\geq np +\eps n\right]\leq \exp\left(-\frac{\eps^2n}{2p(1-p)}\right)\;.
    \]
\end{fact}
An application of \Cref{fct:hoef}, gives that
\[
\pr\left[\sum_{i=1}^NZ_i\geq n\theta/\pi+n t\right]\leq \exp(-\pi t^2n/\theta)\;.
\]
Choose $t \geq \Omega(\sqrt{\theta/n \log(1/\delta)})$, and the result follows.
\end{proof}

\subsubsection{Putting Everything Together: Proof of \Cref{thm:sd-learning-sphere}}
   First, assume that $n\leq  C d\log\log (d)\log(1/\delta)$, for some large enough absolute constant $C>0$. In this case, even if the algorithm makes a wrong prediction in all the points, the mistake bound will be $n=O(d\log \log (d) \log(1/\delta))$. For the rest of the proof, we assume that $n\geq C d\log \log ( d)\log(1/\delta)$. We use the following algorithm for the initialization process.
    \begin{lemma}[Theorem 2 of \cite{dasgupta2005}]\label{lem:initialization}
    Let $\eps,\delta\in(0,1]$. Consider a stream of data points $\x\tth$ drawn uniformly at random from
the surface of the unit sphere in $\R^d$, and the corresponding labels $y\tth$ are consistent with an LTF $\sign(\wstar\cdot \x)$. There is an algorithm that, if it is applied to this stream of data, then
with probability at least $1-\delta$, after $O(d(\log(1/\eps) +\log(1/\delta)))$ mistakes, we get a halfspace $\w$ with generalization error at most $\eps$.
\end{lemma}
    From \Cref{lem:initialization}, we have that with $O(d\log(1/\delta))$ mistakes, we get with probability at least $1-\delta/2$, a halfspace $\vec w\in \R^d$ with generalization error $1/10$, therefore $\theta(\vec w,\wstar)=\pi/10\leq 1/2$, hence $\cos(\theta(\vec w,\wstar))\geq 1/2$.
Let $T= c' d\log\log (n)\log(1/\delta)$ for some sufficiently large absolute constant $c>0$. We split $X$ into $2k$ subsets $U_1,\ldots, U_{2k}$, with $k=n/(2T)$, so that all the subsets contain at least $N=n/T\geq C/c'$, which we can make $C>0$ to be large enough so that $C/c'$ is greater than a sufficiently large absolute constant. Note that each set is independent of the other.

Let $\vec w^{(0)}=\vec w$ and $\vec u^{(0)}=\vec w$. We analyze first the algorithm \Cref{alg:sd-margin-perceptron-sphere} for $\vec w^{(0)}$. Step 5 of \Cref{alg:sd-margin-perceptron-sphere} runs Margin-Perceptron in each set and goes to the next set when a mistake occurs. Let $\vec w^{(t)}$ be the current hypothesis and $\theta^{(t)}=\theta(\vec w^{(t)},\wstar)$.
From \Cref{lem:update-no-buckets}, conditioned on $\vec w^{(t)}$, we have that if a mistake occurred, then we construct a new vector $\vec w^{(t+1)}$ with $\theta^{(t+1)}=\theta(\vec w^{(t+1)},\wstar)$ so that $\tan\theta^{(t+1)}\leq \tan\theta^{(t)}$ and furthermore with probability at least $2/3$ we have that
\[
\tan(\theta^{(t+1)})\leq \tan^{1-1/(8d)}\theta^{(t)} ({C''}/{N})^{1/(8d)}\;,
\]
where $C''>0$ is an absolute constant.
Denote $\xi_{t}=\tan\theta^{(t)}$, we have that $0\leq \xi_0= \tan\theta^{(0)}\leq 1$ and that $\xi_{t+1}\leq \xi_t$. Hence, we have that
\[
\pr[\xi_{t+1}\leq \xi_t^{1-(1/8d)} ({C''}/{N})^{1/(8d)}|\vec w^{(t)}]\geq 2/3\;.
\]
Therefore, using \Cref{lem:super-linear-convergence}, we get that $\xi_{T}\leq e^2 C''/N$, with probability at least $1-\delta/4$. Therefore $\theta^{T}\leq e^2C''/n=e^2 C'' T/n$.  Therefore, in Step 5a of \Cref{alg:sd-margin-perceptron-sphere}, the algorithm made at most $M_1=2T$ mistakes. Next, we bound the number of mistakes in Step 6a. Note that $U=\cup_{i=k}^{2k}U_i$, contains $\Omega(n)$ samples. From \Cref{lem:generalization}, we have that with probability at least $1-\delta/4$ conditioned on the event that $\theta^{(T)}\leq e^2 C''T/n$, \Cref{alg:sd-margin-perceptron-sphere}, labels the points in $U$, with at most $O( T)$ mistakes. The same arguments show the same for the hypothesis $\vec u^{(T)}$. Therefore, the number of mistakes is at most $O(d\log\log( n))\log(1/\delta)$, with probability at least $1-\delta$. 
Combining the two cases above, we obtain that the number of mistakes is
\[
M = O(d \log(1/\delta)) * \begin{cases}
 O( \log \log n )\text{ if } n \geq C d \log \log d\log(1/\delta) \\
O( \log \log d )\text{ otherwise}
\end{cases}
\]
Finally, we may simplify further the above mistake bound by noticing that
if the number of unlabeled points $n \leq d$ the number of mistakes is 
always at most $d$.  Therefore, the total number of mistakes
is $M = O(d \max(\log \log n, 1)) \log(1/\delta)$.

\section{Self-Directed Learning on Arbitrary Datasets}\label{app:forster}

\begin{Ualgorithm}
	\centering
	\fbox{\parbox{6in}{
 {\bf Input:} An unlabeled dataset $X \subseteq \R^d$.\\
 {\bf Output:} A sequence of labeled data $(\x^{(t)}, z^{(t)})$.
\begin{enumerate}
     \setlength\itemsep{0.1em}
    \item Find subspace $V$ of dimension $k$ so that 
    $|X \cap V| \geq (k/d) ~ n$  and
    $X \cap V$ is in $1/(2d)$-approximate Radially Isotropic Position
    using \Cref{pro:algorithmic-forster-transform}. Set $U = X \cap V$.
     \item Randomly initialize guess $\vec w^{(0)} \sim \mathbb S_k$.
    \item For $t =0,\ldots, 5k \log k$:
    \begin{enumerate}
      \setlength\itemsep{0.1em}
        \item Obtain $U_{\vec w^{(t)}}$ by sorting the points of $U$ in decreasing order of margin from 
        $\vec w^{(t)}$, i.e., $|\vec x^{(i+1)} \cdot \vec w^{(t)}| \leq
        |\vec x^{(i)} \cdot \vec w^{(t)}|$.
        \item Initialize the set of correctly predicted points $C \gets \emptyset$.
        \item For $\x \in U_{\vec w^{(t)}}$:
        \begin{enumerate}
          \setlength\itemsep{0.1em}
            \item Predict the label of $\x$ with $\vec w^{(t)}$.
            \item If the prediction is incorrect, update 
            $\vec w^{(t+1)} \gets \vec w^{(t)} - (\vec w^{(t)} \cdot \x) ~ \x$, add 
            $(\x, -\sgn(\vec w^{(t)}\cdot \x) ) $ to $C$, and exit the inner loop.
            \item If the prediction is correct, add 
            $(\x, \sgn(\vec w^{(t)}\cdot \x) ) $ to $C$. 
        \end{enumerate}
        \item If $|C| \geq |U|/(4k)$ then return $C$ and exit the loop.
    \end{enumerate}
    
\end{enumerate}
    
    }}
    \vspace{0.1cm}
	\caption{A Weak Self-Directed Learner for any set $X$.}
	\label{alg:margin-perceptron-distribution-free}
\end{Ualgorithm}

\subsection{Proof of \Cref{pro:weak-distribution-free-sd}}
\label{ssec:proof-weak-distribution-free-sd}
We restate and prove the following proposition giving a weak,
self-directed learner for arbitrary datasets that does $O(d \log d)$
mistakes and with non-trivial success probability (say above $1\%$) 
labels roughly $\Omega(1/d)$-fraction of $X$.
\begin{proposition}[Weak, Self-Directed Learner for Arbitrary Datasets]
    Let $\mathcal{C}$ be the class of LTFs on $\R^d$ and let $X$ be a set of $n$ unlabeled points in $\R^d$. There exists a universal constant 
    $c$ and an algorithm that runs in $\poly(d, n)$ time, makes 
    $O(d \log d)$ mistakes, and, with probability at least $c$, correctly classifies an $\Omega(1/d)$-fraction of the points of $X$.  
\end{proposition}

We start by defining the ``linear-map plus rescaling'' transformation that puts
the dataset in Radially Isotropic Position.

\begin{definition}[Normalized Linear Transformation]
Let $\vec A \in \R^{d \times d}$ be an invertible matrix.
Given a non-zero vector $\x \in\R^d$, we denote by 
$S_{\vec A}(\vec x) = \vec A \vec x/\|\vec A \vec x\|_2$.
We will also overload notation and, given a set of points $X$, we denote
by $S_{\vec A}(X) =\{ S_{\vec A}(\x): \x \in X\}$.
\end{definition}

We shall use the strongly polynomial time algorithmic result to compute
a Forster transform (or show that one does not exist) given in the recent
work of \cite{DTK22}.

\begin{proposition}[Algorithmic Forster Transform, \citep{DTK22}]
\label{pro:algorithmic-forster-transform}
There exists an algorithm, that given a set of points $X$ in 
$\mathbb Z^d \setminus \{\vec 0\} $and some
    $\delta > 0$, runs in time $\poly(n, d, \log(1/\delta))$ and returns a subspace $V$ of $\R^d$
    containing at least a $\dim(V)/d$-fraction of the
    points $X$ and an invertible matrix $\vec A \in \R^{d \times d}$ such that $S_{\vec A}(X\cap V)$ is 
    in $\delta$-approximate radially isotropic position.
\end{proposition}

In the next lemma we show that a dataset in (approximate) Radially Isotropic
Position, satisfies a notion of ``soft-margin'' in the sense that non-trivial
part of the dataset has non-trivial margin with respect to every halfspace.

\begin{lemma}[Soft-Margin via Radially Isotropic Position]
\label{lem:soft-margin-distribution-free}
Let $X$ be a multi-set of non-zero points 
in $1/(2d)$-approximate Radially Isotropic Position.  
For every unit vector $\vec u \in \R^d$, we have 
$\pr_{\x \sim X}[|\vec u \cdot \x| \geq {1}/{(2 \sqrt{d})}] \geq  {1}/{(4 d)})\,.
$
\end{lemma}

\begin{proof} 
Since the set $X$ is in Radially Isotropic Position, we have that $\|\x\|_2 \leq 1$
for every $\x \in X$, and therefore, by Cauchy-Schwarz, $|\vec u \cdot \x| \leq 1$.
For a random variable $z$ taking values in $[0,1]$, the following reverse Markov inequality holds, $\pr[z \geq a] \geq \E[z] - a$ (see, e.g., Appendix B1 in \cite{SB14}).  We obtain that
\[
\pr_{\x \in X}
\left[|\vec u \cdot \x| \geq \frac{1}{2 \sqrt{d}} \right]
=
\pr_{\x \in X}
\left[(\vec u \cdot \x)^2 \geq \frac{1}{4 d} \right]
\geq 
\E_{\x \sim X}[(\vec u \cdot \x)^2] - \frac{1}{4d} 
\geq \frac{1}{2d} - \frac{1}{4d} = \frac{1}{4d} \,,
\]
where we used the fact that the set $X$ is in $1/(2 d)$-approximate Radially Isotropic Position to replace $\E_{\x \sim X}[(\vec u \cdot \x)^2]$ by its lower
bound $1/d - 1/(2d) = 1/(2d)$.
\end{proof}

 Denote by $N$ the number of points that are returned in Step 1 of \Cref{alg:margin-perceptron-distribution-free}, and note that $N\geq nk/d$. From Lemma 3.2.4 in \cite{Ver18}, we get that with probability larger than an absolute constant, the random initialitation gives a point $\vec w^{(0)}$, so that $\vec w^{(0)}\cdot \vec v\geq 1/(2\sqrt k)$. We start the analysis of our algorithm. We show that if \Cref{alg:margin-perceptron-distribution-free} terminates, then $1/(4d)$-fraction of points is correctly classified.  Note that \Cref{alg:margin-perceptron-distribution-free} terminates if the algorithm makes $5d \log d$ mistakes or when $|C|\geq n/(4d)$, i.e., the algorithm classified $1/d$-fraction of points correctly which is the goal of the algorithm. The only bad event is if the algorithm terminated after $5d \log d$ mistakes and $|C|<n/(4d)$.  We argue that this cannot happen. Let $n_i$ be the remaining points in $i$ iteration. Note that $N=n_i+|C|$. We make use of the following lemma (a variant of which was shown in \cite{DunaganV04}).

\begin{lemma}[Margin Perceptron \citep{DunaganV04} ] \label{lem:margin-perceptron}
    Let $\vec v, \vec w^{(0)}\in \R^d$ be unit vectors such that $\vec v\cdot \vec w^{(0)}\geq \alpha$, for some $\alpha>0$. Assume the following: $\vec w^{(t+1)}\gets \vec w^{(t)}-\x^{(t)}(\x^{(t)}\cdot \vec w^{(t)})$ and let $t_0\in Z_+$, so that for all $t\in \Z_+$ with $t\leq t_0$, $|\x^{(t)}\cdot \vec w^{(t)}|\geq  \beta\|\vec w^{(t)}\|_2$ and $(\x^{(t)}\cdot \vec w^{(t)}) (\x^{(t)}\cdot \vec v)<0$. Then, $t_0\leq (2/\beta^2)\log(1/\alpha)$.
\end{lemma}
\begin{proof}
  From our assumption, we have that 
    $\w^{(0)}\cdot \vec v\geq \alpha$. We have that
    \[
    \w^{(t+1)}\cdot\vec v=(\w\tth-(\w \tth\cdot\x\tth)\x\tth)\cdot \vec v =\w\tth\cdot\vec v -(\w
    \tth\cdot\x\tth)(\x\tth\cdot \vec v )\geq\w\tth\cdot\vec v\;,
    \]
      where we used that $(\w \tth\cdot\x\tth)(\x\tth\cdot \vec v )\leq 0$. 
    Therefore, for all $\vec w\tth$, we have inner product with the target vector at least as large as the initialization, i.e., $\w\tth\cdot\vec v\geq \alpha$.  
    We show that the norm of $\w$ decreases multiplicatively.
    We have that
    \begin{align*}
        \|\w\tth-(\w\tth \cdot\x\tth)\x\tth\|_2^2&=\|\w\tth\|_2^2-(\w\tth \cdot\x\tth)^2
        \\&\leq \|\w\tth\|_2^2(1-\beta^2)\;.
    \end{align*}
    Hence, after $t$ iterations, we have that $\|\w\tth\|_2\leq (1-\beta^2)^{t/2}\leq \exp(-t\beta ^2/2)$. If $t\geq (2/\beta^2)\log(1/\alpha)$, we would have $\vec w\tth\cdot \vec v/\|\w\tth\|_2>1$, which is a contradiction. Hence, after $t=(2/\beta^2)\log(1/\alpha)$ updates, we have that either $|\x^{(t)}\cdot \vec w^{(t)}|\leq  \beta\|\vec w^{(t)}\|_2$ or $(\x^{(t)}\cdot \vec w^{(t)}) (\x^{(t)}\cdot \vec v)\geq 0$.
\end{proof}
Assume that after the $t_1=(5 d \log d -1)$ mistake, $|C|<n/(4d)$. That means $n_t=N-|C|\geq n(k/d-1/(4d))\geq N/2$, as $d\geq1$. Let $\mathcal S_t=\{\x^{(i)}: |\vec w^{(t)}\cdot\x^{(i)} |\geq 1/(2\sqrt{k})\}$. From \Cref{lem:soft-margin-distribution-free}, we have that for each $t$, we have $|\mathcal S_t|\geq N/(4k)$ and combining with the fact that $n_t\geq N/2$, that means that either in each iteration, the algorithm makes no mistakes in the set $\mathcal S_t$, which means that $|C|\geq N/(4k)$ and the algorithm terminates, or that it makes one mistake in the set $\mathcal S_t$, which means that if $\x^{(t)}$ is the vector that $\vec w^{(t)}$ made a mistake then $|\vec w\tth \cdot \x\tth|\geq 1/(2\sqrt{k})$. Hence, condition to the event that the algorithm did not terminate before the iteration $t_0$, then by \Cref{lem:margin-perceptron} if $t_0\geq 5d\log d$, then $\vec w^{(t_0)}$ makes no mistakes in the set $\mathcal S_{t_0}$, so it classifies correctly $N/(4k)$ points, and the algorithm terminates. To derive the result, note that $N/(4k)\geq n/(4d)$ by definition.

\subsection{Boosting: Obtaining Strong Self-Directed Learners from Weak Learners}
\label{ssec:boosting}
In this section we present our boosting result showing that given
a weak self-directed learner that labels some non-trivial part of the 
dataset one can  obtain a strong self-directed that labels arbitrarily large
fractions of the dataset.
\begin{lemma}[Boosting]
\label{lem:boosting}
Let $\mathcal{A}$ be a distribution-free self-directed learner 
that makes $M$ mistakes and correctly labels a $(1-\alpha)$-fraction of $X$ 
for some fixed $\alpha \in (0,1)$, with probability at least $c\in(0,1)$.  Then, there exists a strong 
self-directed learner that makes $\wt{O}((M/c) ~ \log(1/(\delta \epsilon)) / \log(1/\alpha) )$ 
mistakes and labels $(1-\epsilon)$-fraction of $X$
with probability at least $1-\delta$.
\end{lemma}
\begin{proof}
We first boost the success probability of the self-directed learner
to $1-\delta$ by repeating the algorithm $\log(1/\delta)$ times.
When we perform independent runs of the algorithm, we stop
a run if the algorithms make more than $M$ mistakes.  
By the assumption that the algorithm succeeds in labeling at least $(1-\alpha)$-fraction with probability at least $c$,
in each run, we obtain that after $O((1/c) \log(1/\delta'))$ runs one of them will 
succeed with probability at least $1-\delta'$.
Therefore, we can boost the success probability of the algorithm to $1-\delta'$
by doing $O( (M/c) \log(1/\delta'))$ mistakes.
After performing a single successful run of the algorithm we have that the number 
of remaining unlabeled data is $\alpha n$.  Similarly, after $k$ runs
the number of unlabeled data is going to be at most $\alpha^k n$.
In order for the fraction of unlabeled data to become smaller than $\epsilon n$
we have to pick $k = \log(1/\epsilon)/\log(1/\alpha)$.  Therefore, 
in order to have probability of success above $1-\delta$ overall, we can do a union bound 
over the $k$ repetitions of the algorithm in order to cover $1-\epsilon$-fraction of the data.
Therefore, we have to pick the success probability $\delta'$ of each run of the algorithm to be $\delta' = 1/(k\delta)$. We conclude that the overhead in the number of mistakes is a factor of $O(\log(1/\delta) (\log\log(1/\epsilon) - \log \log(1/\alpha)))$.
\end{proof}
\subsubsection{Proof of \Cref{thm:strong-distribution-free-sd}}
We restate and prove  \Cref{thm:strong-distribution-free-sd} below.
\begin{theorem}
    Let $\mathcal{C}$ be the class of LTFs on $\R^d$ and let $X$ be a set of $n$ unlabeled points in $\R^d$. There exists a algorithm that runs in
    $\poly(d, n)$ time, makes 
    $\wt{O}(d^2 \log( d / (\epsilon \delta) ) )$ mistakes, and, with probability at least $1-\delta$, correctly classifies a $(1-\epsilon)$-fraction of the points of $X$.  
\end{theorem}
\begin{proof}
    From \Cref{pro:weak-distribution-free-sd}, we get that there is a weak learner so that with probability at least $c$, for some absolute constant $c>0$, classifies $C/d$-fraction of the points of $X$, where $C>0$ is an absolute constant and makes $O(d\log d)$ mistakes. Applying \Cref{lem:boosting} on this algorithm, we get that the total number of mistakes is $\wt{O}(d \log(d/(\delta\eps)) ) /\log(1/(1-1/d) )$. Using the inequality $\log(1+x)\leq x$ for $x>-1$, we get that the total number of mistakes is $\wt{O}(d^2\log(1/(\delta\eps)))$.
\end{proof}
 
\end{document}